\pdfoutput=1
\documentclass[11pt]{article}
\usepackage[letterpaper,margin=1in]{geometry}
\usepackage[T1]{fontenc}
\usepackage{amsmath,amsthm,mathtools}
\usepackage{newtxtext,newtxmath}
\usepackage{microtype}
\usepackage{booktabs,multirow,array,tabularx}
\usepackage{graphicx}
\usepackage[dvipsnames]{xcolor}
\usepackage[round,authoryear]{natbib}
\usepackage{hyperref}
\hypersetup{colorlinks=true,linkcolor=MidnightBlue,citecolor=MidnightBlue,urlcolor=MidnightBlue}
\usepackage[capitalise,noabbrev]{cleveref}
\usepackage{enumitem}
\usepackage{caption}
\setlist{itemsep=2pt,topsep=3pt}
\newtheorem{lemma}{Lemma}
\newtheorem{proposition}{Proposition}
\newtheorem{theorem}{Theorem}
\theoremstyle{remark}
\newtheorem{remark}{Remark}

\makeatletter
\newcommand{\setnum}[2]{\@namedef{num@#1}{#2}}
\newcommand{\nm}[1]{\@ifundefined{num@#1}{\textcolor{red}{\textbf{??#1??}}}{\@nameuse{num@#1}}}
\newcommand{\ci}[1]{[\nm{#1.lo}, \nm{#1.hi}]}
\newcommand{\pmci}[1]{\ensuremath{\nm{#1}^{+\nm{#1.up}}_{-\nm{#1.dn}}}}
\makeatother
\setnum{att.FULL.late.C}{42.11}
\setnum{att.FULL.late.C_att}{1.64}
\setnum{att.FULL.late.C_pure}{40.47}
\setnum{att.FULL.late.H}{3.12}
\setnum{att.FULL.late.attempt}{5.55}
\setnum{att.FULLb.s1.late.C}{40.78}
\setnum{att.FULLb.s1.late.C_att}{16.29}
\setnum{att.FULLb.s1.late.C_pure}{24.49}
\setnum{att.FULLb.s1.late.H}{21.09}
\setnum{att.FULLb.s1.late.attempt}{43.01}
\setnum{att.FULLb.s2.late.C}{39.84}
\setnum{att.FULLb.s2.late.C_att}{0.27}
\setnum{att.FULLb.s2.late.C_pure}{39.57}
\setnum{att.FULLb.s2.late.H}{1.37}
\setnum{att.FULLb.s2.late.attempt}{1.80}
\setnum{att.FULLb.s3.late.C}{32.50}
\setnum{att.FULLb.s3.late.C_att}{5.31}
\setnum{att.FULLb.s3.late.C_pure}{27.19}
\setnum{att.FULLb.s3.late.H}{8.36}
\setnum{att.FULLb.s3.late.attempt}{15.08}
\setnum{att.FULLr.s1.late.C}{43.91}
\setnum{att.FULLr.s1.late.C_att}{0.08}
\setnum{att.FULLr.s1.late.C_pure}{43.83}
\setnum{att.FULLr.s1.late.H}{0.00}
\setnum{att.FULLr.s1.late.attempt}{0.12}
\setnum{att.FULLr.s2.late.C}{39.10}
\setnum{att.FULLr.s2.late.C_att}{0.00}
\setnum{att.FULLr.s2.late.C_pure}{39.10}
\setnum{att.FULLr.s2.late.H}{0.00}
\setnum{att.FULLr.s2.late.attempt}{0.00}
\setnum{att.FULLr.s3.late.C}{35.16}
\setnum{att.FULLr.s3.late.C_att}{0.00}
\setnum{att.FULLr.s3.late.C_pure}{35.16}
\setnum{att.FULLr.s3.late.H}{0.00}
\setnum{att.FULLr.s3.late.attempt}{0.00}
\setnum{att.IB.s1.late.C}{40.35}
\setnum{att.IB.s1.late.C_att}{38.59}
\setnum{att.IB.s1.late.C_pure}{1.76}
\setnum{att.IB.s1.late.H}{45.62}
\setnum{att.IB.s1.late.attempt}{91.68}
\setnum{att.IB.s2.late.C}{38.16}
\setnum{att.IB.s2.late.C_att}{37.19}
\setnum{att.IB.s2.late.C_pure}{0.98}
\setnum{att.IB.s2.late.H}{49.92}
\setnum{att.IB.s2.late.attempt}{90.86}
\setnum{att.IB.s3.late.C}{29.14}
\setnum{att.IB.s3.late.C_att}{28.95}
\setnum{att.IB.s3.late.C_pure}{0.20}
\setnum{att.IB.s3.late.H}{55.78}
\setnum{att.IB.s3.late.attempt}{93.01}
\setnum{att.LB.s1.late.C}{41.84}
\setnum{att.LB.s1.late.C_att}{0.12}
\setnum{att.LB.s1.late.C_pure}{41.72}
\setnum{att.LB.s1.late.H}{0.08}
\setnum{att.LB.s1.late.attempt}{0.39}
\setnum{att.LB.s2.late.C}{38.20}
\setnum{att.LB.s2.late.C_att}{0.08}
\setnum{att.LB.s2.late.C_pure}{38.12}
\setnum{att.LB.s2.late.H}{0.00}
\setnum{att.LB.s2.late.attempt}{0.12}
\setnum{att.LB.s3.late.C}{34.22}
\setnum{att.LB.s3.late.C_att}{0.00}
\setnum{att.LB.s3.late.C_pure}{34.22}
\setnum{att.LB.s3.late.H}{0.16}
\setnum{att.LB.s3.late.attempt}{0.35}
\setnum{att.LOW.late.C}{39.88}
\setnum{att.LOW.late.C_att}{39.65}
\setnum{att.LOW.late.C_pure}{0.23}
\setnum{att.LOW.late.H}{46.37}
\setnum{att.LOW.late.attempt}{93.98}
\setnum{att.LOWm.late.C}{44.84}
\setnum{att.LOWm.late.C_att}{0.00}
\setnum{att.LOWm.late.C_pure}{44.84}
\setnum{att.LOWm.late.H}{0.04}
\setnum{att.LOWm.late.attempt}{0.12}
\setnum{att.NONE.late.C}{36.95}
\setnum{att.NONE.late.C_att}{36.95}
\setnum{att.NONE.late.C_pure}{0.00}
\setnum{att.NONE.late.H}{53.40}
\setnum{att.NONE.late.attempt}{98.28}
\setnum{att.POS.late.C}{40.43}
\setnum{att.POS.late.C_att}{40.16}
\setnum{att.POS.late.C_pure}{0.27}
\setnum{att.POS.late.H}{42.46}
\setnum{att.POS.late.attempt}{89.65}
\setnum{att.RB.s1.late.C}{43.40}
\setnum{att.RB.s1.late.C_att}{42.19}
\setnum{att.RB.s1.late.C_pure}{1.21}
\setnum{att.RB.s1.late.H}{36.33}
\setnum{att.RB.s1.late.attempt}{84.38}
\setnum{att.RB.s2.late.C}{38.87}
\setnum{att.RB.s2.late.C_att}{38.40}
\setnum{att.RB.s2.late.C_pure}{0.47}
\setnum{att.RB.s2.late.H}{51.68}
\setnum{att.RB.s2.late.attempt}{94.53}
\setnum{att.RB.s3.late.C}{32.58}
\setnum{att.RB.s3.late.C_att}{32.58}
\setnum{att.RB.s3.late.C_pure}{0.00}
\setnum{att.RB.s3.late.H}{45.00}
\setnum{att.RB.s3.late.attempt}{83.67}
\setnum{cc.AO.Clate}{43.91}
\setnum{cc.AO.auc}{5.53}
\setnum{cc.AO.hlate}{0.04}
\setnum{cc.AO.labels}{1322}
\setnum{cc.AO.lenlate}{1148}
\setnum{cc.CAG.Clate}{45.31}
\setnum{cc.CAG.auc}{6.72}
\setnum{cc.CAG.hlate}{0.00}
\setnum{cc.CAG.labels}{648}
\setnum{cc.FULL.s1.Clate}{48.05}
\setnum{cc.FULL.s1.auc}{4.43}
\setnum{cc.FULL.s1.hlate}{0.12}
\setnum{cc.FULL.s1.labels}{5089}
\setnum{cc.FULL.s2.labels}{4931}
\setnum{cc.FULL.s3.labels}{4962}
\setnum{cc.FULLLB.s1.dClate}{\ensuremath{+}4.80}
\setnum{cc.FULLLB.s1.dClate.dn}{2.30}
\setnum{cc.FULLLB.s1.dClate.up}{2.31}
\setnum{cc.FULLLB.s1.dauc}{\ensuremath{-}2.17}
\setnum{cc.FULLLB.s1.dauc.dn}{0.50}
\setnum{cc.FULLLB.s1.dauc.up}{0.50}
\setnum{cc.FULLLB.s2.dClate}{\ensuremath{+}1.76}
\setnum{cc.FULLLB.s2.dClate.dn}{1.92}
\setnum{cc.FULLLB.s2.dClate.up}{2.07}
\setnum{cc.FULLLB.s2.dauc}{\ensuremath{-}1.52}
\setnum{cc.FULLLB.s2.dauc.dn}{0.47}
\setnum{cc.FULLLB.s2.dauc.up}{0.46}
\setnum{cc.FULLLB.s3.dClate}{\ensuremath{-}0.78}
\setnum{cc.FULLLB.s3.dClate.dn}{1.99}
\setnum{cc.FULLLB.s3.dClate.up}{2.03}
\setnum{cc.FULLLB.s3.dauc}{\ensuremath{-}0.21}
\setnum{cc.FULLLB.s3.dauc.dn}{0.52}
\setnum{cc.FULLLB.s3.dauc.up}{0.51}
\setnum{cc.IPW.s1.Clate}{38.71}
\setnum{cc.IPW.s1.auc}{6.45}
\setnum{cc.IPW.s1.hlate}{0.82}
\setnum{cc.IPW.s1.labels}{1341}
\setnum{cc.IPW.s2.labels}{1355}
\setnum{cc.IPW.s3.labels}{1225}
\setnum{cc.IPWstd.Clate}{42.54}
\setnum{cc.IPWstd.auc}{20.00}
\setnum{cc.IPWstd.hlate}{19.06}
\setnum{cc.IPWstd.labels}{1832}
\setnum{cc.IPWstd_IPW.dauc}{\ensuremath{+}13.55}
\setnum{cc.IPWstd_IPW.dauc.dn}{1.12}
\setnum{cc.IPWstd_IPW.dauc.up}{1.08}
\setnum{cc.IPWtrue.Clate}{36.56}
\setnum{cc.IPWtrue.auc}{19.40}
\setnum{cc.IPWtrue.hlate}{14.73}
\setnum{cc.IPWtrue.labels}{1660}
\setnum{cc.LB.max.lenlate}{847}
\setnum{cc.LB.min.lenlate}{749}
\setnum{cc.LB.s1.Clate}{43.24}
\setnum{cc.LB.s1.auc}{6.60}
\setnum{cc.LB.s1.hlate}{0.08}
\setnum{cc.LB.s1.labels}{1403}
\setnum{cc.LB.s1.lenlate}{749}
\setnum{cc.LB.s2.labels}{1393}
\setnum{cc.LB.s3.labels}{1395}
\setnum{cc.LBFULL.min.labelratio}{28}
\setnum{cc.LBIPW.labels.excess}{6.9}
\setnum{cc.LBIPW.max.dClate.abs}{11.13}
\setnum{cc.LBIPW.mean.dClate}{\ensuremath{+}7.80}
\setnum{cc.LBIPW.mean.dauc}{\ensuremath{-}2.65}
\setnum{cc.LBIPW.mean.dhlate}{\ensuremath{-}0.62}
\setnum{cc.LBIPW.mean.dpass8}{\ensuremath{+}5.21}
\setnum{cc.LBIPW.min.dClate.abs}{4.53}
\setnum{cc.LBIPW.p.dClate}{0.055}
\setnum{cc.LBIPW.p.dauc}{0.236}
\setnum{cc.LBIPW.p.dhlate}{0.058}
\setnum{cc.LBIPW.p.dpass8}{0.116}
\setnum{cc.LBIPW.s1.dClate}{\ensuremath{+}4.53}
\setnum{cc.LBIPW.s1.dClate.dn}{2.15}
\setnum{cc.LBIPW.s1.dClate.up}{2.23}
\setnum{cc.LBIPW.s1.dauc}{\ensuremath{+}0.16}
\setnum{cc.LBIPW.s1.dauc.dn}{0.57}
\setnum{cc.LBIPW.s1.dauc.up}{0.57}
\setnum{cc.LBIPW.s1.dhlate}{\ensuremath{-}0.74}
\setnum{cc.LBIPW.s1.dhlate.dn}{0.39}
\setnum{cc.LBIPW.s1.dhlate.up}{0.35}
\setnum{cc.LBIPW.s1.dpass8}{\ensuremath{+}2.81}
\setnum{cc.LBIPW.s2.dClate}{\ensuremath{+}11.13}
\setnum{cc.LBIPW.s2.dClate.dn}{2.34}
\setnum{cc.LBIPW.s2.dClate.up}{2.35}
\setnum{cc.LBIPW.s2.dauc}{\ensuremath{-}5.32}
\setnum{cc.LBIPW.s2.dauc.dn}{0.72}
\setnum{cc.LBIPW.s2.dauc.up}{0.68}
\setnum{cc.LBIPW.s2.dhlate}{\ensuremath{-}0.82}
\setnum{cc.LBIPW.s2.dhlate.dn}{0.39}
\setnum{cc.LBIPW.s2.dhlate.up}{0.35}
\setnum{cc.LBIPW.s2.dpass8}{\ensuremath{+}9.06}
\setnum{cc.LBIPW.s3.dClate}{\ensuremath{+}7.73}
\setnum{cc.LBIPW.s3.dClate.dn}{2.34}
\setnum{cc.LBIPW.s3.dClate.up}{2.35}
\setnum{cc.LBIPW.s3.dauc}{\ensuremath{-}2.77}
\setnum{cc.LBIPW.s3.dauc.dn}{0.64}
\setnum{cc.LBIPW.s3.dauc.up}{0.62}
\setnum{cc.LBIPW.s3.dhlate}{\ensuremath{-}0.31}
\setnum{cc.LBIPW.s3.dhlate.dn}{0.28}
\setnum{cc.LBIPW.s3.dhlate.up}{0.23}
\setnum{cc.LBIPW.s3.dpass8}{\ensuremath{+}3.75}
\setnum{cc.LBIPW.t.dClate}{4.09}
\setnum{cc.LBnostd.Clate}{54.30}
\setnum{cc.LBnostd.auc}{5.80}
\setnum{cc.LBnostd.hlate}{0.00}
\setnum{cc.LBnostd.labels}{1559}
\setnum{cc.LBnostd.s3.hlate}{0.12}
\setnum{cc.LBnostdIPW.s2.dClate}{\ensuremath{+}17.03}
\setnum{cc.LBnostdIPW.s2.dClate.abs}{17.03}
\setnum{cc.LBnostdIPW.s2.dClate.dn}{3.20}
\setnum{cc.LBnostdIPW.s2.dClate.up}{3.13}
\setnum{cc.LBnostdIPW.s3.dClate}{\ensuremath{+}12.19}
\setnum{cc.LBnostdIPW.s3.dClate.abs}{12.19}
\setnum{cc.LBnostdIPW.s3.dClate.dn}{2.46}
\setnum{cc.LBnostdIPW.s3.dClate.up}{2.54}
\setnum{cc.LBnostdLB.max.dClate.abs}{11.05}
\setnum{cc.LBnostdLB.mean.dClate}{\ensuremath{+}7.14}
\setnum{cc.LBnostdLB.min.dClate.abs}{4.45}
\setnum{cc.LBnostdLB.p.dClate}{0.071}
\setnum{cc.LBnostdLB.s1.dClate}{\ensuremath{+}11.05}
\setnum{cc.LBnostdLB.s1.labelratio.abs}{11}
\setnum{cc.LBnostdLB.s1.tokenratio.abs}{18}
\setnum{cc.LBnostdLB.s2.dClate}{\ensuremath{+}5.90}
\setnum{cc.LBnostdLB.s3.dClate}{\ensuremath{+}4.45}
\setnum{cc.LBnostdLB.s3.labelratio.abs}{5}
\setnum{cc.LBnostdLB.s3.tokenratio.abs}{30}
\setnum{cc.PAC.Clate}{18.09}
\setnum{cc.PAC.auc}{43.53}
\setnum{cc.PAC.hlate}{64.80}
\setnum{cc.PAC.labels}{2006}
\setnum{cc.RB.Clate}{40.39}
\setnum{cc.RB.auc}{40.11}
\setnum{cc.RB.hlate}{50.12}
\setnum{cc.RB.labels}{2368}
\setnum{cc.RBLB.dClate}{\ensuremath{-}2.85}
\setnum{cc.RBLB.dClate.dn}{2.74}
\setnum{cc.RBLB.dClate.up}{2.58}
\setnum{cc.RBLB.dauc}{\ensuremath{+}33.51}
\setnum{cc.RBLB.dauc.abs}{33.51}
\setnum{cc.RBLB.dauc.dn}{2.01}
\setnum{cc.RBLB.dauc.hi}{\ensuremath{+}35.40}
\setnum{cc.RBLB.dauc.lo}{\ensuremath{+}31.50}
\setnum{cc.RBLB.dauc.up}{1.89}
\setnum{cc.RBLB.labelratio.abs}{69}
\setnum{cc.RawA.Clate}{26.45}
\setnum{cc.RawA.auc}{22.79}
\setnum{cc.RawA.hlate}{15.20}
\setnum{cc.RawA.labels}{1639}
\setnum{cc.advIPW.Clate}{39.77}
\setnum{cc.advIPW.auc}{8.36}
\setnum{cc.advIPW.hlate}{2.19}
\setnum{cc.advIPW.labels}{1429}
\setnum{cc.advIPW.max.hlate}{2.19}
\setnum{cc.advIPW.max.lenlate}{1117}
\setnum{cc.advIPW.min.hlate}{0.35}
\setnum{cc.advIPW.min.lenlate}{677}
\setnum{cc.advIPWLB.max.dClate.abs}{9.92}
\setnum{cc.advIPWLB.min.dClate.abs}{3.48}
\setnum{cc.advIPWLB.s1.dClate}{\ensuremath{-}3.48}
\setnum{cc.advIPWLB.s1.dClate.dn}{2.22}
\setnum{cc.advIPWLB.s1.dClate.up}{2.15}
\setnum{cc.advIPWLB.s1.dauc}{\ensuremath{+}1.76}
\setnum{cc.advIPWLB.s1.dauc.dn}{0.66}
\setnum{cc.advIPWLB.s1.dauc.up}{0.62}
\setnum{cc.advIPWLB.s2.dClate}{\ensuremath{-}9.38}
\setnum{cc.advIPWLB.s2.dClate.dn}{2.34}
\setnum{cc.advIPWLB.s2.dClate.up}{2.35}
\setnum{cc.advIPWLB.s2.dauc}{\ensuremath{+}5.65}
\setnum{cc.advIPWLB.s2.dauc.dn}{0.80}
\setnum{cc.advIPWLB.s2.dauc.up}{0.81}
\setnum{cc.advIPWLB.s3.dClate}{\ensuremath{-}9.92}
\setnum{cc.advIPWLB.s3.dClate.dn}{2.11}
\setnum{cc.advIPWLB.s3.dClate.up}{2.07}
\setnum{cc.advIPWLB.s3.dauc}{\ensuremath{+}6.31}
\setnum{cc.advIPWLB.s3.dauc.dn}{0.75}
\setnum{cc.advIPWLB.s3.dauc.up}{0.81}
\setnum{cc.optIPW.Clate}{44.65}
\setnum{cc.optIPW.auc}{7.90}
\setnum{cc.optIPW.hlate}{0.27}
\setnum{cc.optIPW.labels}{1498}
\setnum{cc.optIPWLB.dClate}{\ensuremath{+}1.41}
\setnum{cc.optIPWLB.dClate.dn}{2.19}
\setnum{cc.optIPWLB.dClate.up}{2.03}
\setnum{cc.optIPWLB.dauc}{\ensuremath{+}1.30}
\setnum{cc.optIPWLB.dauc.dn}{0.68}
\setnum{cc.optIPWLB.dauc.up}{0.66}
\setnum{cc.optIPWadvIPW.dClate}{\ensuremath{+}4.88}
\setnum{cc.optIPWadvIPW.dClate.dn}{2.07}
\setnum{cc.optIPWadvIPW.dClate.up}{2.00}
\setnum{cc.start.h}{25.4}
\setnum{diag.batches}{8}
\setnum{diag.cos}{0.50}
\setnum{diag.dSH}{\ensuremath{-}0.0428}
\setnum{diag.dSH.se}{0.0063}
\setnum{diag.dfromH}{\ensuremath{-}0.0039}
\setnum{diag.dfromH.se}{0.0009}
\setnum{diag.dfromRest}{\ensuremath{+}0.0042}
\setnum{diag.dfromRest.se}{0.0010}
\setnum{diag.dsgdH}{\ensuremath{+}0.0003}
\setnum{diag.dsgdH.se}{0.0004}
\setnum{diag.online.net}{\ensuremath{-}0.0011}
\setnum{diag.perm.mean}{\ensuremath{+}0.0002}
\setnum{diag.perm.pairs}{64}
\setnum{diag.perm.ratio}{2.2}
\setnum{diag.perm.rms}{0.0024}
\setnum{diag2.direct}{\ensuremath{-}0.0003}
\setnum{diag2.direct.se}{0.0003}
\setnum{diag2.proj.max}{\ensuremath{+}63}
\setnum{diag2.proj.min}{\ensuremath{+}27}
\setnum{diag2.reweight}{\ensuremath{+}0.0006}
\setnum{diag2.reweight.se}{0.0003}
\setnum{diagp1.dSH}{\ensuremath{-}0.0193}
\setnum{diagp1.dSH.se}{0.0066}
\setnum{diagp1.dfromH}{\ensuremath{-}0.0002}
\setnum{diagp1.dfromH.se}{0.0003}
\setnum{diagp1.dfromRest}{\ensuremath{-}0.0002}
\setnum{diagp1.dfromRest.se}{0.0005}
\setnum{diagp1.direct}{\ensuremath{+}0.0004}
\setnum{diagp1.direct.se}{0.0006}
\setnum{diagp1.dsgdH}{\ensuremath{-}0.0004}
\setnum{diagp1.dsgdH.se}{0.0006}
\setnum{diagp1.online.net}{\ensuremath{+}0.0020}
\setnum{diagp1.perm.ratio}{0.85}
\setnum{diagp1.perm.rms}{0.0017}
\setnum{diagp1.proj.max}{\ensuremath{+}3}
\setnum{diagp1.proj.min}{\ensuremath{-}11}
\setnum{diagp1.reweight}{\ensuremath{-}0.0008}
\setnum{diagp1.reweight.se}{0.0002}
\setnum{diff.FULLb.emp_easiest.C}{98.08}
\setnum{diff.FULLb.emp_easiest.H}{0.17}
\setnum{diff.FULLb.emp_easiest.attempt}{9.64}
\setnum{diff.FULLb.emp_hardest.C}{60.56}
\setnum{diff.FULLb.emp_hardest.H}{4.46}
\setnum{diff.FULLb.emp_hardest.attempt}{15.60}
\setnum{diff.FULLbgap.emp.H}{\ensuremath{+}4.29}
\setnum{diff.FULLbgap.emp.H.hi}{\ensuremath{+}5.40}
\setnum{diff.FULLbgap.emp.H.lo}{\ensuremath{+}3.28}
\setnum{diff.FULLr.emp_hardest.C}{63.02}
\setnum{diff.FULLr.emp_hardest.H}{0.01}
\setnum{diff.FULLr.emp_hardest.attempt}{0.04}
\setnum{diff.LB.emp_easiest.C}{98.31}
\setnum{diff.LB.emp_easiest.H}{0.00}
\setnum{diff.LB.emp_easiest.attempt}{0.02}
\setnum{diff.LB.emp_hardest.C}{60.37}
\setnum{diff.LB.emp_hardest.H}{0.04}
\setnum{diff.LB.emp_hardest.attempt}{0.19}
\setnum{diff.LB.label_hard.C}{68.28}
\setnum{diff.LB.label_hard.H}{0.00}
\setnum{diff.LB.label_hard.attempt}{0.05}
\setnum{diff.RB.emp_easiest.C}{97.47}
\setnum{diff.RB.emp_easiest.H}{1.90}
\setnum{diff.RB.emp_easiest.attempt}{98.60}
\setnum{diff.RB.emp_hardest.C}{60.27}
\setnum{diff.RB.emp_hardest.H}{31.83}
\setnum{diff.RB.emp_hardest.attempt}{94.00}
\setnum{diff.RB.label_hard.C}{65.29}
\setnum{diff.RB.label_hard.H}{23.44}
\setnum{diff.RB.label_hard.attempt}{89.91}
\setnum{diff.RBgap.emp.H}{\ensuremath{+}29.93}
\setnum{diff.RBgap.emp.H.hi}{\ensuremath{+}34.40}
\setnum{diff.RBgap.emp.H.lo}{\ensuremath{+}25.69}
\setnum{diff.size.emp_easiest}{112}
\setnum{diff.size.emp_hardest}{144}
\setnum{diff.size.label_hard}{44}
\setnum{diff.start.emp_easiest.C}{100.00}
\setnum{diff.start.emp_easiest.H}{0.00}
\setnum{diff.start.emp_easiest.attempt}{15.12}
\setnum{diff.start.emp_hardest.C}{51.95}
\setnum{diff.start.emp_hardest.H}{7.68}
\setnum{diff.start.emp_hardest.attempt}{30.30}
\setnum{ho.cc.AO.M}{34.4}
\setnum{ho.cc.CAG-LB.C}{\ensuremath{+}1.0}
\setnum{ho.cc.CAG-LB.C.dn}{1.3}
\setnum{ho.cc.CAG-LB.C.up}{1.4}
\setnum{ho.cc.LBFULL.s1.dC}{\ensuremath{-}2.1}
\setnum{ho.cc.LBFULL.s1.dC.dn}{1.5}
\setnum{ho.cc.LBFULL.s1.dC.hi}{\ensuremath{-}0.7}
\setnum{ho.cc.LBFULL.s1.dC.lo}{\ensuremath{-}3.6}
\setnum{ho.cc.LBFULL.s1.dC.up}{1.4}
\setnum{ho.cc.LBFULL.s2.dC}{\ensuremath{-}0.0}
\setnum{ho.cc.LBFULL.s2.dC.abs}{0.0}
\setnum{ho.cc.LBFULL.s2.dC.dn}{1.3}
\setnum{ho.cc.LBFULL.s2.dC.hi}{\ensuremath{+}1.2}
\setnum{ho.cc.LBFULL.s2.dC.lo}{\ensuremath{-}1.3}
\setnum{ho.cc.LBFULL.s2.dC.up}{1.2}
\setnum{ho.cc.LBFULL.s3.dC}{\ensuremath{+}2.2}
\setnum{ho.cc.LBFULL.s3.dC.dn}{1.2}
\setnum{ho.cc.LBFULL.s3.dC.hi}{\ensuremath{+}3.4}
\setnum{ho.cc.LBFULL.s3.dC.lo}{\ensuremath{+}1.0}
\setnum{ho.cc.LBFULL.s3.dC.up}{1.2}
\setnum{ho.cc.LBRB.dC}{\ensuremath{+}5.3}
\setnum{ho.cc.LBRB.dC.abs}{5.3}
\setnum{ho.cc.LBRB.dC.dn}{1.4}
\setnum{ho.cc.LBRB.dC.hi}{\ensuremath{+}6.7}
\setnum{ho.cc.LBRB.dC.lo}{\ensuremath{+}3.9}
\setnum{ho.cc.LBRB.dC.up}{1.4}
\setnum{ho.cc.LBRB.dC_parseable}{\ensuremath{+}5.2}
\setnum{ho.cc.LBRB.dC_parseable.dn}{1.5}
\setnum{ho.cc.LBRB.dC_parseable.up}{1.4}
\setnum{ho.cc.LBadvIPW.max.dC}{11.5}
\setnum{ho.cc.LBadvIPW.max23.dC_parseable}{10.9}
\setnum{ho.cc.LBadvIPW.min.dC}{8.0}
\setnum{ho.cc.LBadvIPW.min23.dC_parseable}{10.0}
\setnum{ho.cc.LBadvIPW.s1.dC}{\ensuremath{+}8.0}
\setnum{ho.cc.LBadvIPW.s1.dC.dn}{1.3}
\setnum{ho.cc.LBadvIPW.s1.dC.up}{1.3}
\setnum{ho.cc.LBadvIPW.s1.dC_parseable}{\ensuremath{-}5.2}
\setnum{ho.cc.LBadvIPW.s1.dC_parseable.dn}{1.6}
\setnum{ho.cc.LBadvIPW.s1.dC_parseable.up}{1.7}
\setnum{ho.cc.LBadvIPW.s1.dM}{\ensuremath{-}23.2}
\setnum{ho.cc.LBadvIPW.s1.dM.dn}{1.6}
\setnum{ho.cc.LBadvIPW.s1.dM.up}{1.7}
\setnum{ho.cc.LBadvIPW.s2.dC}{\ensuremath{+}11.5}
\setnum{ho.cc.LBadvIPW.s2.dC.dn}{1.5}
\setnum{ho.cc.LBadvIPW.s2.dC.up}{1.6}
\setnum{ho.cc.LBadvIPW.s2.dC_parseable}{\ensuremath{+}10.9}
\setnum{ho.cc.LBadvIPW.s2.dC_parseable.dn}{1.5}
\setnum{ho.cc.LBadvIPW.s2.dC_parseable.up}{1.6}
\setnum{ho.cc.LBadvIPW.s2.dM}{\ensuremath{-}1.9}
\setnum{ho.cc.LBadvIPW.s2.dM.dn}{0.5}
\setnum{ho.cc.LBadvIPW.s2.dM.up}{0.6}
\setnum{ho.cc.LBadvIPW.s3.dC}{\ensuremath{+}10.0}
\setnum{ho.cc.LBadvIPW.s3.dC.dn}{1.4}
\setnum{ho.cc.LBadvIPW.s3.dC.up}{1.4}
\setnum{ho.cc.LBadvIPW.s3.dC_parseable}{\ensuremath{+}10.0}
\setnum{ho.cc.LBadvIPW.s3.dC_parseable.dn}{1.3}
\setnum{ho.cc.LBadvIPW.s3.dC_parseable.up}{1.4}
\setnum{ho.cc.LBadvIPW.s3.dM}{\ensuremath{-}0.0}
\setnum{ho.cc.LBadvIPW.s3.dM.dn}{0.2}
\setnum{ho.cc.LBadvIPW.s3.dM.up}{0.1}
\setnum{ho.cc.LBoptIPW.dC}{\ensuremath{+}1.2}
\setnum{ho.cc.LBoptIPW.dC.abs}{1.2}
\setnum{ho.cc.LBoptIPW.dC.dn}{1.3}
\setnum{ho.cc.LBoptIPW.dC.hi}{\ensuremath{+}2.5}
\setnum{ho.cc.LBoptIPW.dC.lo}{\ensuremath{-}0.1}
\setnum{ho.cc.LBoptIPW.dC.up}{1.3}
\setnum{ho.cc.RB.H}{48.8}
\setnum{ho.cc.advIPW.s1.M}{23.4}
\setnum{ho.cc.advIPW.s2.M}{3.8}
\setnum{ho.cc.advIPW.s3.M}{0.3}
\setnum{ho.cc.optIPWadvIPW.dC}{\ensuremath{+}6.8}
\setnum{ho.cc.optIPWadvIPW.dC.abs}{6.8}
\setnum{ho.cc.optIPWadvIPW.dC.dn}{1.4}
\setnum{ho.cc.optIPWadvIPW.dC.hi}{\ensuremath{+}8.1}
\setnum{ho.cc.optIPWadvIPW.dC.lo}{\ensuremath{+}5.4}
\setnum{ho.cc.optIPWadvIPW.dC.up}{1.3}
\setnum{ho.cc.optIPWadvIPW.dM}{\ensuremath{-}22.8}
\setnum{ho.cc.optIPWadvIPW.dM.dn}{1.6}
\setnum{ho.cc.optIPWadvIPW.dM.up}{1.7}
\setnum{ho.lc.start.H}{14.9}
\setnum{ho3.cc.IPW.trunc.max}{13.9}
\setnum{ho3.cc.IPW.trunc.min}{3.3}
\setnum{ho3.cc.LB.trunc.max}{1.6}
\setnum{ho3.cc.LB.trunc.min}{0.2}
\setnum{ho3.cc.LBIPW.max.C.abs}{11.8}
\setnum{ho3.cc.LBIPW.max.Cparse.abs}{5.9}
\setnum{ho3.cc.LBIPW.mean.C}{\ensuremath{+}8.4}
\setnum{ho3.cc.LBIPW.mean.Cparse}{\ensuremath{+}5.1}
\setnum{ho3.cc.LBIPW.min.C.abs}{6.1}
\setnum{ho3.cc.LBIPW.min.Cparse.abs}{4.5}
\setnum{ho3.cc.LBIPW.p.C}{0.040}
\setnum{ho3.cc.LBIPW.s1.Cparse}{\ensuremath{+}4.5}
\setnum{ho3.cc.LBIPW.s1.Cparse.dn}{1.4}
\setnum{ho3.cc.LBIPW.s1.Cparse.up}{1.3}
\setnum{ho3.cc.LBIPW.s2.Cparse}{\ensuremath{+}5.9}
\setnum{ho3.cc.LBIPW.s2.Cparse.dn}{1.5}
\setnum{ho3.cc.LBIPW.s2.Cparse.up}{1.4}
\setnum{ho3.cc.LBIPW.s3.Cparse}{\ensuremath{+}5.0}
\setnum{ho3.cc.LBIPW.s3.Cparse.dn}{1.4}
\setnum{ho3.cc.LBIPW.s3.Cparse.up}{1.4}
\setnum{ho3.cc.LBIPW.t.C}{4.85}
\setnum{ho3.cc.s1.LBIPW.C}{\ensuremath{+}6.1}
\setnum{ho3.cc.s1.LBIPW.C.dn}{1.4}
\setnum{ho3.cc.s1.LBIPW.C.up}{1.3}
\setnum{ho3.cc.s2.LBIPW.C}{\ensuremath{+}11.8}
\setnum{ho3.cc.s2.LBIPW.C.dn}{1.3}
\setnum{ho3.cc.s2.LBIPW.C.up}{1.4}
\setnum{ho3.cc.s3.LB.M}{0.2}
\setnum{ho3.cc.s3.LBIPW.C}{\ensuremath{+}7.3}
\setnum{ho3.cc.s3.LBIPW.C.dn}{1.2}
\setnum{ho3.cc.s3.LBIPW.C.up}{1.4}
\setnum{ho3.lc.LBRB.mean.C}{\ensuremath{-}0.2}
\setnum{ho3.lc.LBRB.mean.H}{\ensuremath{-}53.9}
\setnum{ho3.lc.LBRB.p.C}{0.871}
\setnum{ho3.lc.LBRB.p.H}{0.006}
\setnum{ho3.lc.s1.LB.H}{0.3}
\setnum{ho3.lc.s1.LBRB.C}{\ensuremath{-}1.3}
\setnum{ho3.lc.s1.LBRB.C.dn}{1.6}
\setnum{ho3.lc.s1.LBRB.C.up}{1.5}
\setnum{ho3.lc.s1.LBRB.H}{\ensuremath{-}50.4}
\setnum{ho3.lc.s1.LBRB.H.dn}{5.6}
\setnum{ho3.lc.s1.LBRB.H.up}{5.5}
\setnum{ho3.lc.s1.RB.H}{50.8}
\setnum{ho3.lc.s2.LB.H}{0.1}
\setnum{ho3.lc.s2.LBRB.C}{\ensuremath{+}1.5}
\setnum{ho3.lc.s2.LBRB.C.dn}{1.9}
\setnum{ho3.lc.s2.LBRB.C.up}{2.0}
\setnum{ho3.lc.s2.LBRB.H}{\ensuremath{-}62.2}
\setnum{ho3.lc.s2.LBRB.H.dn}{5.1}
\setnum{ho3.lc.s2.LBRB.H.up}{5.4}
\setnum{ho3.lc.s2.RB.H}{62.2}
\setnum{ho3.lc.s3.LB.H}{0.1}
\setnum{ho3.lc.s3.LBRB.C}{\ensuremath{-}0.7}
\setnum{ho3.lc.s3.LBRB.C.dn}{1.6}
\setnum{ho3.lc.s3.LBRB.C.up}{1.8}
\setnum{ho3.lc.s3.LBRB.H}{\ensuremath{-}49.1}
\setnum{ho3.lc.s3.LBRB.H.dn}{5.4}
\setnum{ho3.lc.s3.LBRB.H.up}{5.3}
\setnum{ho3.lc.s3.RB.H}{49.1}
\setnum{ho3.nostd.s2.LBnostd.C}{54.6}
\setnum{ho3.nostd.s2.LBnostdIPW.C}{\ensuremath{+}14.0}
\setnum{ho3.nostd.s2.LBnostdIPW.C.abs}{14.0}
\setnum{ho3.nostd.s2.LBnostdIPW.C.dn}{1.5}
\setnum{ho3.nostd.s2.LBnostdIPW.C.up}{1.6}
\setnum{ho3.nostd.s2.LBnostdLB.C}{\ensuremath{+}2.2}
\setnum{ho3.nostd.s2.LBnostdLB.C.abs}{2.2}
\setnum{ho3.nostd.s2.LBnostdLB.C.dn}{1.1}
\setnum{ho3.nostd.s2.LBnostdLB.C.up}{1.1}
\setnum{ho3.nostd.s3.LBnostd.C}{54.2}
\setnum{ho3.nostd.s3.LBnostd.M}{18.7}
\setnum{ho3.nostd.s3.LBnostdIPW.C}{\ensuremath{+}9.4}
\setnum{ho3.nostd.s3.LBnostdIPW.C.abs}{9.4}
\setnum{ho3.nostd.s3.LBnostdIPW.C.dn}{1.3}
\setnum{ho3.nostd.s3.LBnostdIPW.C.up}{1.3}
\setnum{ho3.nostd.s3.LBnostdLB.C}{\ensuremath{+}2.0}
\setnum{ho3.nostd.s3.LBnostdLB.C.abs}{2.0}
\setnum{ho3.nostd.s3.LBnostdLB.C.dn}{1.0}
\setnum{ho3.nostd.s3.LBnostdLB.C.up}{1.1}
\setnum{ho4k.cc.IPW.continued}{1779}
\setnum{ho4k.cc.IPW.rescued}{9}
\setnum{ho4k.cc.LBIPW.max.dC}{11.7}
\setnum{ho4k.cc.LBIPW.min.dC}{6.1}
\setnum{ho4k.cc.LBIPW.s1.dC}{\ensuremath{+}6.1}
\setnum{ho4k.cc.LBIPW.s1.dC.dn}{1.4}
\setnum{ho4k.cc.LBIPW.s1.dC.up}{1.3}
\setnum{ho4k.cc.LBIPW.s2.dC}{\ensuremath{+}11.7}
\setnum{ho4k.cc.LBIPW.s2.dC.dn}{1.3}
\setnum{ho4k.cc.LBIPW.s2.dC.up}{1.4}
\setnum{ho4k.cc.LBIPW.s3.dC}{\ensuremath{+}7.3}
\setnum{ho4k.cc.LBIPW.s3.dC.dn}{1.3}
\setnum{ho4k.cc.LBIPW.s3.dC.up}{1.4}
\setnum{ho4k.cc.LBadvIPW.s1.dC}{\ensuremath{+}7.9}
\setnum{ho4k.cc.LBadvIPW.s1.dC.abs}{7.9}
\setnum{ho4k.cc.LBadvIPW.s1.dC.dn}{1.3}
\setnum{ho4k.cc.LBadvIPW.s1.dC.hi}{\ensuremath{+}9.2}
\setnum{ho4k.cc.LBadvIPW.s1.dC.lo}{\ensuremath{+}6.6}
\setnum{ho4k.cc.LBadvIPW.s1.dC.up}{1.3}
\setnum{ho4k.cc.LBadvIPW.s2.dC}{\ensuremath{+}11.5}
\setnum{ho4k.cc.LBadvIPW.s2.dC.dn}{1.5}
\setnum{ho4k.cc.LBadvIPW.s2.dC.up}{1.6}
\setnum{ho4k.cc.LBadvIPW.s3.dC}{\ensuremath{+}10.0}
\setnum{ho4k.cc.LBadvIPW.s3.dC.dn}{1.4}
\setnum{ho4k.cc.LBadvIPW.s3.dC.up}{1.4}
\setnum{ho4k.cc.advIPW.s1.continued}{1863}
\setnum{ho4k.cc.advIPW.s1.rescued}{7}
\setnum{lc.D0.mean.dauc}{\ensuremath{+}1.21}
\setnum{lc.D0.p.dauc}{0.206}
\setnum{lc.D0.s1.dClate}{\ensuremath{-}1.58}
\setnum{lc.D0.s1.dClate.dn}{1.51}
\setnum{lc.D0.s1.dClate.up}{1.46}
\setnum{lc.D0.s1.dauc}{\ensuremath{+}1.09}
\setnum{lc.D0.s1.dauc.dn}{0.86}
\setnum{lc.D0.s1.dauc.up}{0.79}
\setnum{lc.D0.s1.dhlate}{\ensuremath{+}1.99}
\setnum{lc.D0.s1.dhlate.dn}{1.79}
\setnum{lc.D0.s1.dhlate.up}{1.74}
\setnum{lc.D0.s2.dClate}{\ensuremath{-}1.27}
\setnum{lc.D0.s2.dClate.dn}{1.41}
\setnum{lc.D0.s2.dClate.up}{1.43}
\setnum{lc.D0.s2.dauc}{\ensuremath{+}0.14}
\setnum{lc.D0.s2.dauc.dn}{0.94}
\setnum{lc.D0.s2.dauc.up}{0.94}
\setnum{lc.D0.s2.dhlate}{\ensuremath{-}3.16}
\setnum{lc.D0.s2.dhlate.dn}{2.11}
\setnum{lc.D0.s2.dhlate.up}{2.03}
\setnum{lc.D0.s3.dClate}{\ensuremath{-}1.60}
\setnum{lc.D0.s3.dClate.dn}{1.51}
\setnum{lc.D0.s3.dClate.up}{1.42}
\setnum{lc.D0.s3.dauc}{\ensuremath{+}2.40}
\setnum{lc.D0.s3.dauc.dn}{0.89}
\setnum{lc.D0.s3.dauc.up}{0.82}
\setnum{lc.D0.s3.dhlate}{\ensuremath{+}3.30}
\setnum{lc.D0.s3.dhlate.dn}{1.84}
\setnum{lc.D0.s3.dhlate.up}{1.93}
\setnum{lc.D4.s1.dClate}{\ensuremath{+}2.07}
\setnum{lc.D4.s1.dClate.dn}{2.38}
\setnum{lc.D4.s1.dClate.up}{2.34}
\setnum{lc.D4.s1.dauc}{\ensuremath{-}28.90}
\setnum{lc.D4.s1.dauc.dn}{2.23}
\setnum{lc.D4.s1.dauc.up}{2.13}
\setnum{lc.D4.s1.dhlate}{\ensuremath{-}39.26}
\setnum{lc.D4.s1.dhlate.dn}{4.45}
\setnum{lc.D4.s1.dhlate.up}{4.38}
\setnum{lc.D4.s2.dClate}{\ensuremath{+}1.21}
\setnum{lc.D4.s2.dClate.dn}{1.91}
\setnum{lc.D4.s2.dClate.up}{1.95}
\setnum{lc.D4.s2.dauc}{\ensuremath{-}29.08}
\setnum{lc.D4.s2.dauc.dn}{2.07}
\setnum{lc.D4.s2.dauc.up}{2.02}
\setnum{lc.D4.s2.dhlate}{\ensuremath{-}31.72}
\setnum{lc.D4.s2.dhlate.dn}{3.98}
\setnum{lc.D4.s2.dhlate.up}{3.87}
\setnum{lc.FULL.Clate}{42.11}
\setnum{lc.FULL.auc}{12.37}
\setnum{lc.FULL.hlate}{3.12}
\setnum{lc.FULL.labels}{5161}
\setnum{lc.FULLb.pool.attempt}{19.96}
\setnum{lc.FULLb.pool.residue}{9.69}
\setnum{lc.FULLbFULLr.mean.d_att}{\ensuremath{+}19.92}
\setnum{lc.FULLbFULLr.p.d_att}{0.242}
\setnum{lc.FULLbFULLr.s1.datt}{\ensuremath{+}42.89}
\setnum{lc.FULLbFULLr.s1.datt.dn}{4.69}
\setnum{lc.FULLbFULLr.s1.datt.up}{4.77}
\setnum{lc.FULLbFULLr.s2.datt}{\ensuremath{+}1.80}
\setnum{lc.FULLbFULLr.s2.datt.dn}{0.94}
\setnum{lc.FULLbFULLr.s2.datt.up}{1.17}
\setnum{lc.FULLbFULLr.s3.datt}{\ensuremath{+}15.08}
\setnum{lc.FULLbFULLr.s3.datt.dn}{3.36}
\setnum{lc.FULLbFULLr.s3.datt.up}{3.40}
\setnum{lc.IB.s1.hlate}{45.62}
\setnum{lc.IB.s3.hlate}{55.78}
\setnum{lc.IBLB.max.d_auc.abs}{41.25}
\setnum{lc.IBLB.mean.d_auc}{\ensuremath{+}40.68}
\setnum{lc.IBLB.mean.d_c}{\ensuremath{-}2.20}
\setnum{lc.IBLB.min.d_auc.abs}{40.31}
\setnum{lc.IBLB.p.d_auc}{\ensuremath{<}0.001}
\setnum{lc.IBLB.p.d_c}{0.280}
\setnum{lc.IBLB.s1.dClate}{\ensuremath{-}1.48}
\setnum{lc.IBLB.s1.dClate.dn}{2.11}
\setnum{lc.IBLB.s1.dClate.up}{2.10}
\setnum{lc.IBLB.s1.dauc}{\ensuremath{+}41.25}
\setnum{lc.IBLB.s1.dauc.dn}{2.42}
\setnum{lc.IBLB.s1.dauc.up}{2.42}
\setnum{lc.IBLB.s2.dClate}{\ensuremath{-}0.04}
\setnum{lc.IBLB.s2.dClate.dn}{1.91}
\setnum{lc.IBLB.s2.dClate.up}{1.99}
\setnum{lc.IBLB.s2.dauc}{\ensuremath{+}40.31}
\setnum{lc.IBLB.s2.dauc.dn}{2.26}
\setnum{lc.IBLB.s2.dauc.up}{2.33}
\setnum{lc.IBLB.s3.dClate}{\ensuremath{-}5.08}
\setnum{lc.IBLB.s3.dClate.dn}{2.54}
\setnum{lc.IBLB.s3.dClate.up}{2.38}
\setnum{lc.IBLB.s3.dauc}{\ensuremath{+}40.47}
\setnum{lc.IBLB.s3.dauc.dn}{2.41}
\setnum{lc.IBLB.s3.dauc.up}{2.41}
\setnum{lc.IBLB.t.d_auc}{140.3}
\setnum{lc.IBRB.s1.dauc}{\ensuremath{+}10.36}
\setnum{lc.IBRB.s1.dauc.dn}{1.47}
\setnum{lc.IBRB.s1.dauc.up}{1.39}
\setnum{lc.IBRB.s2.dauc}{\ensuremath{-}1.22}
\setnum{lc.IBRB.s2.dauc.dn}{1.12}
\setnum{lc.IBRB.s2.dauc.up}{1.20}
\setnum{lc.IBRB.s3.dauc}{\ensuremath{+}4.08}
\setnum{lc.IBRB.s3.dauc.dn}{1.55}
\setnum{lc.IBRB.s3.dauc.up}{1.56}
\setnum{lc.LB.s1.Clate}{41.84}
\setnum{lc.LB.s1.auc}{3.04}
\setnum{lc.LB.s1.hit5}{62.5}
\setnum{lc.LB.s1.hlate}{0.08}
\setnum{lc.LB.s1.labels}{1385}
\setnum{lc.LB.s2.auc}{1.84}
\setnum{lc.LB.s2.hit5}{55.4}
\setnum{lc.LB.s2.labels}{1314}
\setnum{lc.LB.s3.auc}{3.22}
\setnum{lc.LB.s3.hit5}{56.5}
\setnum{lc.LB.s3.labels}{1359}
\setnum{lc.LB120.hlate}{31.25}
\setnum{lc.LBIB.mean.labelratio}{51}
\setnum{lc.LBRB.max.d_auc.abs}{30.89}
\setnum{lc.LBRB.mean.d_auc}{\ensuremath{-}36.27}
\setnum{lc.LBRB.mean.d_c}{\ensuremath{-}0.20}
\setnum{lc.LBRB.mean.d_h}{\ensuremath{-}44.26}
\setnum{lc.LBRB.mean.d_p16}{\ensuremath{-}0.62}
\setnum{lc.LBRB.mean.labelratio}{53}
\setnum{lc.LBRB.min.d_auc.abs}{41.53}
\setnum{lc.LBRB.p.d_auc}{0.007}
\setnum{lc.LBRB.p.d_c}{0.857}
\setnum{lc.LBRB.p.d_h}{0.010}
\setnum{lc.LBRB.p.d_p16}{0.580}
\setnum{lc.LBRB.p.labelratio}{0.002}
\setnum{lc.LBRB.s1.dClate}{\ensuremath{-}1.56}
\setnum{lc.LBRB.s1.dClate.dn}{2.35}
\setnum{lc.LBRB.s1.dClate.up}{2.26}
\setnum{lc.LBRB.s1.dauc}{\ensuremath{-}30.89}
\setnum{lc.LBRB.s1.dauc.dn}{2.30}
\setnum{lc.LBRB.s1.dauc.up}{2.26}
\setnum{lc.LBRB.s1.dhlate}{\ensuremath{-}36.25}
\setnum{lc.LBRB.s1.dhlate.dn}{4.69}
\setnum{lc.LBRB.s1.dhlate.up}{4.57}
\setnum{lc.LBRB.s1.dpass16}{\ensuremath{+}0.00}
\setnum{lc.LBRB.s2.dClate}{\ensuremath{-}0.66}
\setnum{lc.LBRB.s2.dClate.dn}{2.00}
\setnum{lc.LBRB.s2.dClate.up}{1.99}
\setnum{lc.LBRB.s2.dauc}{\ensuremath{-}41.53}
\setnum{lc.LBRB.s2.dauc.dn}{2.44}
\setnum{lc.LBRB.s2.dauc.up}{2.33}
\setnum{lc.LBRB.s2.dhlate}{\ensuremath{-}51.68}
\setnum{lc.LBRB.s2.dhlate.dn}{4.96}
\setnum{lc.LBRB.s2.dhlate.up}{4.96}
\setnum{lc.LBRB.s2.dpass16}{\ensuremath{-}2.50}
\setnum{lc.LBRB.s3.dClate}{\ensuremath{+}1.64}
\setnum{lc.LBRB.s3.dClate.dn}{1.68}
\setnum{lc.LBRB.s3.dClate.up}{1.80}
\setnum{lc.LBRB.s3.dauc}{\ensuremath{-}36.39}
\setnum{lc.LBRB.s3.dauc.dn}{2.18}
\setnum{lc.LBRB.s3.dauc.up}{2.22}
\setnum{lc.LBRB.s3.dhlate}{\ensuremath{-}44.84}
\setnum{lc.LBRB.s3.dhlate.dn}{4.77}
\setnum{lc.LBRB.s3.dhlate.up}{4.72}
\setnum{lc.LBRB.s3.dpass16}{\ensuremath{+}0.63}
\setnum{lc.LBRB.t.d_auc}{\ensuremath{-}11.80}
\setnum{lc.LBp10.s1.hlate}{40.2}
\setnum{lc.LBp10.s1.labels}{801}
\setnum{lc.LBp10.s2.hlate}{45.0}
\setnum{lc.LBp10h.s1.hlate}{7.9}
\setnum{lc.LBp10h.s1.lastbin}{10.2}
\setnum{lc.LBp10h.s1.minbin}{2.0}
\setnum{lc.LBp10hc40.bin2}{44.9}
\setnum{lc.LBp10hc40.hlate}{40.5}
\setnum{lc.LBp20.max.labelratio}{63}
\setnum{lc.LBp20.min.labelratio}{61}
\setnum{lc.LBp20.s1.hlate}{0.12}
\setnum{lc.LBp20.s2.hlate}{0.27}
\setnum{lc.LBscale10.hlate}{0.04}
\setnum{lc.LBscale10.labelratio}{33}
\setnum{lc.LBscale10.labels}{456}
\setnum{lc.LBscale10LBm.dClate}{\ensuremath{+}0.04}
\setnum{lc.LBscale10LBm.dClate.dn}{2.07}
\setnum{lc.LBscale10LBm.dClate.up}{2.11}
\setnum{lc.LBscale10LBm.dauc}{\ensuremath{+}0.33}
\setnum{lc.LBscale10LBm.dauc.dn}{0.47}
\setnum{lc.LBscale10LBm.dauc.up}{0.48}
\setnum{lc.LBscale10advIPWp10.dClate}{\ensuremath{-}2.23}
\setnum{lc.LBscale10advIPWp10.dClate.dn}{1.68}
\setnum{lc.LBscale10advIPWp10.dClate.up}{1.64}
\setnum{lc.LBscale10advIPWp10.dauc}{\ensuremath{-}1.56}
\setnum{lc.LBscale10advIPWp10.dauc.dn}{0.56}
\setnum{lc.LBscale10advIPWp10.dauc.up}{0.54}
\setnum{lc.LOW.Clate}{39.88}
\setnum{lc.LOW.auc}{42.62}
\setnum{lc.LOW.hlate}{46.37}
\setnum{lc.LOW.labels}{3348}
\setnum{lc.LOW_POS.dauc}{\ensuremath{+}1.27}
\setnum{lc.LOW_POS.dauc.dn}{1.35}
\setnum{lc.LOW_POS.dauc.up}{1.30}
\setnum{lc.LOWm.Clate}{44.84}
\setnum{lc.LOWm.auc}{2.52}
\setnum{lc.LOWm.hlate}{0.04}
\setnum{lc.LOWm.labels}{1779}
\setnum{lc.LOWm_FULL.dClate}{\ensuremath{+}2.73}
\setnum{lc.LOWm_FULL.dClate.dn}{2.30}
\setnum{lc.LOWm_FULL.dClate.up}{2.27}
\setnum{lc.LOWm_FULL.labelratio}{34}
\setnum{lc.LOWm_LOW.dClate.abs}{4.96}
\setnum{lc.LOWm_LOW.dauc}{\ensuremath{-}40.10}
\setnum{lc.LOWm_LOW.dauc.dn}{2.38}
\setnum{lc.LOWm_LOW.dauc.up}{2.42}
\setnum{lc.LOWm_LOW.labelratio}{53}
\setnum{lc.NONE.Clate}{36.95}
\setnum{lc.NONE.auc}{50.28}
\setnum{lc.NONE.hlate}{53.40}
\setnum{lc.NONE.labels}{0}
\setnum{lc.ON120.hlate}{47.81}
\setnum{lc.ON120_LB120.dauc}{\ensuremath{+}4.90}
\setnum{lc.ON120_LB120.dauc.abs}{4.90}
\setnum{lc.ON120_LB120.dauc.dn}{1.77}
\setnum{lc.ON120_LB120.dauc.up}{1.87}
\setnum{lc.ON80_LB.dauc}{\ensuremath{+}6.42}
\setnum{lc.ON80_LB.dauc.abs}{6.42}
\setnum{lc.ON80_LB.dauc.dn}{1.17}
\setnum{lc.ON80_LB.dauc.up}{1.30}
\setnum{lc.POS.Clate}{40.43}
\setnum{lc.POS.auc}{41.35}
\setnum{lc.POS.hlate}{42.46}
\setnum{lc.POS.labels}{4069}
\setnum{lc.RB.s1.hit5}{45.9}
\setnum{lc.RB.s1.hlate}{36.33}
\setnum{lc.RB.s1.labels}{2451}
\setnum{lc.RB.s2.hit5}{45.5}
\setnum{lc.RB.s2.hlate}{51.68}
\setnum{lc.RB.s2.labels}{2714}
\setnum{lc.RB.s3.hit5}{43.9}
\setnum{lc.RB.s3.labels}{2561}
\setnum{lc.advIPW.s1.dClate}{\ensuremath{+}0.0}
\setnum{lc.advIPW.s1.dClate.dn}{2.5}
\setnum{lc.advIPW.s1.dClate.up}{2.4}
\setnum{lc.advIPW.s1.dauc}{\ensuremath{-}1.1}
\setnum{lc.advIPW.s1.dauc.dn}{0.5}
\setnum{lc.advIPW.s1.dauc.up}{0.4}
\setnum{lc.advIPW.s1.hlate}{0.12}
\setnum{lc.advIPW.s1.hoC}{\ensuremath{-}1.0}
\setnum{lc.advIPW.s1.hoC.dn}{1.4}
\setnum{lc.advIPW.s1.hoC.up}{1.4}
\setnum{lc.advIPW.s1.labels}{1470}
\setnum{lc.advIPW.s1.labels.LB}{1385}
\setnum{lc.advIPW.s2.dClate}{\ensuremath{+}2.5}
\setnum{lc.advIPW.s2.dClate.dn}{2.2}
\setnum{lc.advIPW.s2.dClate.up}{2.3}
\setnum{lc.advIPW.s2.dauc}{\ensuremath{+}5.2}
\setnum{lc.advIPW.s2.dauc.dn}{0.8}
\setnum{lc.advIPW.s2.dauc.up}{0.9}
\setnum{lc.advIPW.s2.hlate}{3.44}
\setnum{lc.advIPW.s2.hoC}{\ensuremath{-}0.8}
\setnum{lc.advIPW.s2.hoC.dn}{1.1}
\setnum{lc.advIPW.s2.hoC.up}{1.1}
\setnum{lc.advIPW.s2.labels}{1435}
\setnum{lc.advIPW.s2.labels.LB}{1314}
\setnum{lc.advIPW.s3.dClate}{\ensuremath{-}2.2}
\setnum{lc.advIPW.s3.dClate.dn}{2.3}
\setnum{lc.advIPW.s3.dClate.up}{2.2}
\setnum{lc.advIPW.s3.dauc}{\ensuremath{-}1.5}
\setnum{lc.advIPW.s3.dauc.dn}{0.4}
\setnum{lc.advIPW.s3.dauc.up}{0.4}
\setnum{lc.advIPW.s3.hlate}{0.16}
\setnum{lc.advIPW.s3.hoC}{\ensuremath{+}0.2}
\setnum{lc.advIPW.s3.hoC.dn}{1.2}
\setnum{lc.advIPW.s3.hoC.up}{1.2}
\setnum{lc.advIPW.s3.labels}{1394}
\setnum{lc.advIPW.s3.labels.LB}{1359}
\setnum{lc.advIPWp10.hlate}{0.82}
\setnum{lc.advIPWp10.labels}{540}
\setnum{lc.start.h}{24.2}
\setnum{unit.c0.EG}{0.1161}
\setnum{unit.c0.EI}{0.1316}
\setnum{unit.c115.EG}{0.0333}
\setnum{unit.c115.EI}{\ensuremath{-}0.1494}

\newcommand{\E}{\mathbb{E}}
\newcommand{\LB}{\textsc{LB}}
\newcommand{\RB}{\textsc{RB}}
\newcommand{\IPW}{\textsc{IPW}}
\newcommand{\FULL}{\textsc{Full}}
\newcommand{\NONE}{\textsc{None}}
\newcommand{\PAC}{\textsc{PAC}}
\newcommand{\pp}{\,pp}

\title{Audit Where Hacks Pay: Budgeted Strong-Verifier Audits\\ Against Reward Hacking in GRPO}
\author{%
Shurun Li\textsuperscript{1} \qquad Michael Hua Wang\textsuperscript{1} \qquad Haibo Zeng\textsuperscript{2} \qquad Long Wang\textsuperscript{1}\\[6pt]
\normalsize\textsuperscript{1}REASONS Lab, Institute for Network Sciences and Cyberspace, Tsinghua University\\
\normalsize\textsuperscript{2}Virginia Tech\\[4pt]
\small\texttt{\{lisr22,\,wanghua24\}@mails.tsinghua.edu.cn}, \texttt{hbzeng@vt.edu}, \texttt{longwang@tsinghua.edu.cn}%
}
\date{Preprint, October 2026}

\begin{document}
\maketitle

\begin{abstract}
Reinforcement learning with verifiable rewards trains on cheap checks, such as unit tests, that a policy can learn to game, while a stronger check is too expensive to run on every rollout. We study how to spend a limited strong-verifier budget in GRPO, separating which weakly passing rollouts to audit from what reward a caught hack receives. At a fixed policy state with centered advantages, the expected update depends on a design only through each rollout's marginal audit probability and correction; reflecting a caught hack's reward about the failure reward negates the advantages of an all-hack group; and auditing whole groups with the fewest weak passes first (low-pass budgeted audits, \LB{}) minimizes the advantage mass left on hacks among pass-count designs with a fixed correction when a label's value falls with the pass count. We tested this analysis in actual LLM training: from each run's starting state, the sign of this mass on a frozen panel predicted whether hacking would be suppressed in 33 of 39 runs, including 13 of 16 preregistered ones. In LeetCode (Wong et al., 2026) and Countdown-Code (Khalifa et al., 2026), \LB{} drives hacking to near zero, while random audits with the same penalty and budget do not and use more labels. We also find that the randomness of the update affects training: in Countdown-Code, two unbiased corrections, per-item reward-level IPW and a group-level advantage correction, bring hacking down too but end \nm{ho3.cc.LBIPW.min.C.abs}--\nm{ho3.cc.LBIPW.max.C.abs} and \nm{ho.cc.LBadvIPW.min.dC}--\nm{ho.cc.LBadvIPW.max.dC} points lower in held-out accuracy than \LB{} at the same budget. A full audit, which the group-level correction reproduces in expectation, stays within about two points of \LB{} in three seeds, while \LB{} uses \nm{cc.LBFULL.min.labelratio}\% of its labels, and allocating the same correction by pass count recovers most of the loss (one seed). Training outcomes thus depend not only on the expected update but also on the randomness of its estimate.
\end{abstract}

\section{Introduction}
\label{sec:intro}

Reinforcement learning with verifiable rewards (RLVR) trains language models on automatic checks such as unit tests, and a policy can learn to pass a weak check without solving the task. In Countdown-Code it rewrites the target value in its own solution file so that the provided test passes \citep{khalifa2026countdown}; in the LeetCode environment of \citet{wong2026monitor} it defines its own \texttt{run\_tests} function, which the grader then executes. A stronger check, such as the full test suite or a semantic check of the answer, is usually available, but it costs too much to run on every rollout \citep{wang2025solved,sakib2026audit}.

This paper asks how to spend a fixed strong-verifier budget during group-relative policy optimization (GRPO; \citealp{shao2024deepseekmath}). An audit design makes two choices: which weakly passing rollouts to send to the strong verifier (\emph{selection}) and what reward a caught hack receives (\emph{correction}). Two defaults are natural: random audits with an inverse-propensity correction, which reproduce the full-audit reward in expectation \citep{horvitz1952generalization,moya2026verifier,sakib2026audit}, and giving a caught hack the reward its strong verdict implies. Neither accounts for how GRPO turns rewards into updates. A rollout's advantage depends on its siblings, so the same caught hack can push the policy strongly, weakly, or not at all.

We make three observations. First, without group standard-deviation normalization, the expected raw policy gradient depends on an audit design only through each rollout's marginal audit probability and correction value (\cref{lem:marginals}): coupling audits within groups changes the noise of an update, while selection and correction change its direction. Second, giving a caught hack the reward $2b-a$, the pass reward $a$ reflected about the failure reward $b$, exactly negates the advantages of a group whose weak passes are all hacks, whereas the strong-verdict reward $b$ only zeroes them (\cref{prop:reflect}, \cref{fig:theory}b). Third, an uncaught hack's advantage is largest when few siblings pass, so a label spent on a low-pass group buys the most correction (\cref{fig:theory}a). \emph{Low-pass budgeted audits} (\LB{}) therefore audit whole groups in increasing order of weak passes until a fixed fraction $p$ of the batch's weak passes has been checked in expectation (\cref{fig:main}a). Among whole-group designs that select by pass count and give caught hacks a fixed correction, \LB{} minimizes a frozen-state objective, the advantage mass left on hacks, whenever the value of a label does not increase with the pass count (\cref{thm:lb}).

\begin{figure}[t]
\centering
\includegraphics[width=\linewidth]{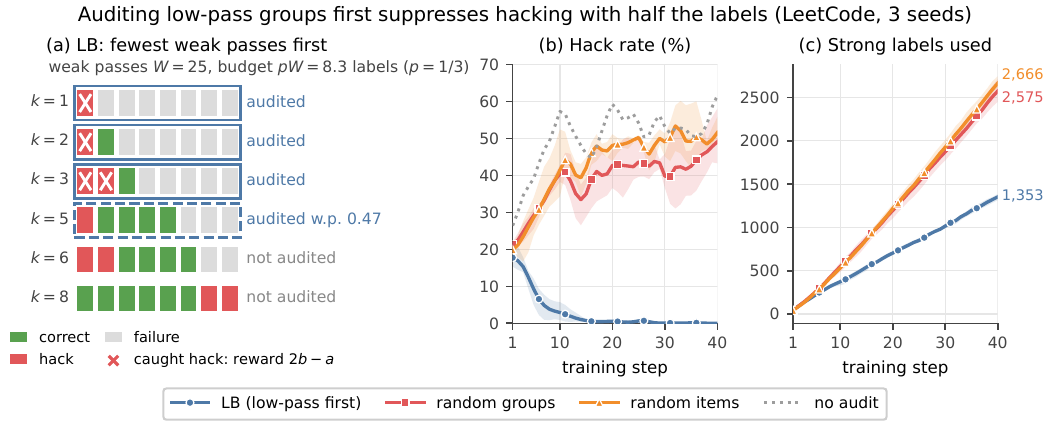}
\caption{Low-pass budgeted audits (\LB{}). (a) Groups of a batch sorted by their number of weak passes $k$: \LB{} audits whole groups, fewest weak passes first, until a fraction $p$ of the weak passes has been checked in expectation (the boundary group is audited with the remaining probability), and a caught hack receives the reflected penalty $2b-a$. (b) LeetCode hack rate with the same penalty and budget rule: \LB{}, auditing whole groups in random order (\RB{}), auditing single weak passes at random, and no audit; mean over three seeds with their range shaded, smoothed with a Gaussian kernel of one step. (c) Strong labels used: because the budget is a fraction of the weak passes, \LB{}'s spending falls as hacking is suppressed.}
\label{fig:main}
\end{figure}

We test these designs in two reward-hacking environments from prior work, LeetCode \citep{wong2026monitor} with LoRA training of Qwen3-4B and Countdown-Code \citep{khalifa2026countdown} with full fine-tuning of Qwen2.5-3B, starting from checkpoints where about a quarter of the training rollouts are hacks. Every comparison was preregistered. Our contributions follow three steps:
\begin{itemize}
\item \textbf{When the push on hacks is optimal} (\cref{sec:theory}). The three observations above characterize how an audit design changes the expected update at a frozen policy state, and \LB{} minimizes the advantage mass left on hacks among whole-group pass-count designs with a fixed correction whenever a label's value does not increase with the pass count (\cref{thm:lb}).
\item \textbf{Predicting training from the frozen push} (\cref{sec:results}). With the same penalty and budget rule, \LB{} drives LeetCode hacking to near zero within about ten steps in all three seeds, while auditing groups in random order leaves late hacking at \nm{lc.RB.s1.hlate}--\nm{lc.RB.s2.hlate}\% and auditing single weak passes at random at \nm{lc.IB.s1.hlate}--\nm{lc.IB.s3.hlate}\%, with \LB{} using about half their labels; random group order also fails in Countdown-Code (one seed; \cref{sec:res-select}). The correction value matters as much: the strong-verdict reward $c=b$ only removes the push on a caught all-hack group, and with it auditing the lowest-pass groups is no better than auditing random groups, whereas with $2b-a$ the same groups remove hacking with half the labels (\cref{sec:res-penalty}). From each run's starting state, the sign of the frozen hack mass $J$ predicted whether hacking would be suppressed in 33 of 39 runs, including 13 of 16 preregistered ones. The push stops predicting in two places: an online rule that improved $J$ on frozen panels trained worse than \LB{}, so the magnitude of $J$ does not rank designs, and with $c=b$ hacking spread beyond the low-pass groups (\cref{sec:boundaries}).
\item \textbf{Beyond the expected update, randomness affects training} (\cref{sec:res-ipw}). In Countdown-Code, per-item reward-level inverse-propensity weighting (\IPW{}) and a group-level advantage correction \citep{sakib2026audit} also suppress hacking, but in all three seeds end lower than \LB{} in training correctness (by \nm{cc.LBIPW.min.dClate.abs}--\nm{cc.LBIPW.max.dClate.abs} and \nm{cc.advIPWLB.min.dClate.abs}--\nm{cc.advIPWLB.max.dClate.abs}\pp{}) and in held-out accuracy (by \nm{ho3.cc.LBIPW.min.C.abs}--\nm{ho3.cc.LBIPW.max.C.abs} and \nm{ho.cc.LBadvIPW.min.dC}--\nm{ho.cc.LBadvIPW.max.dC}\pp{}); in LeetCode the group-level correction was close to \LB{}. A full audit, which every unbiased group-level design matches in expectation on a frozen batch, stayed within about two points of \LB{} on held-out problems in all three seeds while \LB{} used \nm{cc.LBFULL.min.labelratio}\% of its labels, and allocating the group-level correction by pass count closed most of the gap (one seed). Its loss therefore comes from the randomness of the estimate and the allocation of audits, not from the full-audit update it targets.
\end{itemize}

\paragraph{Relation to budgeted audit estimation.} \citet{sakib2026audit} study unbiased estimation of full-suite training signals on fixed software-engineering trajectories, including a direct correction of normalized group advantages, and target groups with more passes to reduce estimator variance. We study how selection and correction affect on-policy learning, and target groups with \emph{fewer} passes first, where an uncaught hack is reinforced the most. Our comparisons cover two unbiased corrections at matched budgets: per-item reward-level \IPW{} and the group-level advantage correction.

\section{Setup}
\label{sec:setup}

\paragraph{Weak and strong verifiers.} For each prompt GRPO samples a group of $K$ responses. The weak verifier gives reward $a$ to a response it accepts (a \emph{weak pass}), $b$ to a well-formed response it rejects, and $0$ to a malformed one, with $0\le b<a$. A strong verifier, run only on weak passes, decides whether a weak pass is genuine. Every strong pass is a weak pass, so responses fall into three classes: correct (C: weak and strong pass), hack (H: weak pass, strong fail), and failure (F: weak fail). In Countdown-Code $(a,b)=(1.2,0.2)$; in LeetCode $(a,b)=(3.5,0.5)$, which is 3.0 for passing the tests plus 0.5 for compiling. We assume the strong verifier is reliable and its errors are independent of the policy; \cref{sec:conclusion} discusses this assumption.

\paragraph{Group-relative advantages.} With rewards $r_1,\dots,r_K$, GRPO uses $A_i=(r_i-\bar r)/(\sigma+\epsilon)$, where $\bar r$ and $\sigma$ are the group mean and unbiased standard deviation (``group std'') and $A_i=0$ when $\sigma=\epsilon=0$, or the centered $A_i=r_i-\bar r$ (``no std''; \citealp{liu2025drgrpo}). In both of our environments, training uses a clipped surrogate \citep{schulman2017ppo} with token-mean aggregation \citep{yu2025dapo}. To first order at the sampling policy, a batch moves the parameters along $g=\sum_i A_i\nabla_\theta\log\pi_\theta(y_i\mid x_i)$, up to positive per-response weights that depend only on the batch.

\paragraph{Audit designs.} An audit design selects a set $M$ of weak passes for the strong verifier and replaces the reward of each caught hack $i\in M\cap H$ by a correction value. The budget is a fraction $p$ of the batch's $W$ weak passes in expectation, $\E|M|=pW$. With a fixed $p$ the spend follows $W$, the number of correct responses plus hacks, so it falls only when weak passes fall. We compare the following selection rules.
\begin{itemize}
\item \LB{} (low-pass budgeted): groups with at least one weak pass are sorted by their number of weak passes $k$, with ties broken at random. Whole groups are audited in this order while they fit in the remaining budget; the first group that does not fit is audited with probability equal to the remaining budget divided by its $k$. The expected cost is exactly $pW$.
\item \RB{} (random budgeted): the same rule with the groups in random order.
\item Random items: each weak pass is audited independently with probability $p$, and a caught hack receives the correction $c$ itself; the expected cost matches \RB{}'s, but audits within a group are not coupled.
\item \IPW{} (per item): each weak pass is audited independently with probability $p$, and a caught hack receives $a+(c-a)/p$, so its expected reward is $c$: the design is unbiased for a full audit with correction $c$ at the reward level, and with centered advantages also for the raw gradient (\cref{lem:marginals}) \citep{horvitz1952generalization}.
\item \FULL{} audits every weak pass; \NONE{} audits nothing. Two fixed-count variants used in LeetCode audit the 8 groups with the fewest weak passes (\textsc{Low}) or 8 random groups with at least one weak pass (\textsc{Pos}) at each step.
\end{itemize}
The correction value is either the strong-verdict reward $c=b$, which treats a caught hack like any well-formed failure, or the \emph{reflected penalty} $c=2b-a$: $-0.8$ in Countdown-Code and $-2.5$ in LeetCode. For \IPW{}, the reflected target gives a caught hack $1.2+(-0.8-1.2)/0.28=-5.94$ in Countdown-Code.

\paragraph{Metrics.} For measurement only, the strong verifier labels every training rollout. The \emph{hack rate} $h_t$ is the fraction of rollouts at step $t$ in class H, and correctness $C_t$ the fraction in class C. We report $\mathrm{AUC}(h)$, the mean of $h_t$ over the 40 training steps; the hack rate, correctness and pass@$K$ over the late window of steps 31--40; the number of strong labels used for training (audited weak passes, summed over training); and the number of generated tokens. Trained checkpoints from the main comparisons and key ablations are also evaluated on held-out test problems. Unless stated otherwise, differences are in percentage points (pp), first arm minus second.

\section{How audit designs change the update}
\label{sec:theory}

We analyze a single training step at the sampling policy: the batch, its weak rewards and its classes are fixed, and the audit design is the only source of randomness (\cref{thm:lb,prop:pure} also average over group compositions given pass counts). Proofs and further results are in \cref{app:proofs,app:theory}.

\subsection{Selection and correction change the direction; coupling changes the noise}

Let $m_i\in\{0,1\}$ indicate that weak pass $i$ is audited, with marginal probability $\pi_i=\Pr(m_i=1)$ given the batch, and let $c_i$ be the reward a caught hack receives ($c_i=c$ for a fixed penalty, $c_i=a+(c-a)/\pi_i$ for \IPW{}), fixed given the batch. The post-audit reward is $r_i=r^w_i+m_i\,\mathbf 1[i\in H]\,(c_i-a)$.

\begin{lemma}[Only marginals matter without std normalization]
\label{lem:marginals}
With centered advantages $A_i=r_i-\bar r$, the expected raw policy gradient at a fixed state, $\E[g]=\sum_i\E[A_i]\,\nabla_\theta\log\pi_\theta(y_i\mid x_i)$, depends on the design only through $\pi_i(c_i-a)$ for each hack $i$. In particular, auditing whole groups and auditing items independently with the same marginals and the same correction give the same $\E[g]$, and \IPW{} with target $c$ gives the same $\E[g]$ as auditing every weak pass with correction $c$.
\end{lemma}

The lemma follows from linearity of the reward correction, the group mean and the sum over responses. With centered advantages, the unit of audit therefore affects only the noise of the raw gradient, which is what \citet{sakib2026audit} measure, while its expectation changes only through \emph{which} responses are likely to be audited and \emph{what} a caught hack receives. Group std normalization breaks the lemma (\cref{sec:theory-ipw,app:unit}); adaptive optimizers and gradient clipping break its corollary for parameter steps, which are nonlinear in $g$.

\subsection{The reflected penalty negates a two-valued group}

\begin{figure}[t]
\centering
\includegraphics[width=0.66\linewidth]{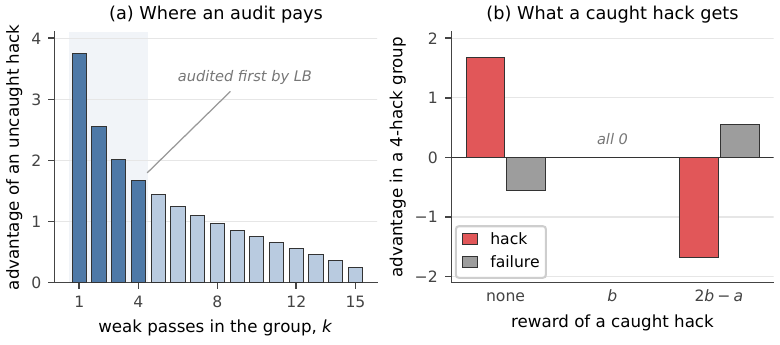}
\caption{The two mechanisms behind \LB{}, for a group of $K=16$ responses with LeetCode rewards ($a=3.5$, $b=0.5$, normalized advantages). (a) The advantage of an uncaught hack among $k$ weak passes, $\sqrt{(K-1)(K-k)/(Kk)}$, is largest when few siblings pass (\cref{prop:pure}). (b) In a group with four hacks and twelve failures, auditing every weak pass with the strong-verdict reward $b$ zeroes all advantages, while the reflected penalty $2b-a$ negates them (\cref{prop:reflect}).}
\label{fig:theory}
\end{figure}

\begin{proposition}[Reflection]
\label{prop:reflect}
Consider a group whose $k$ weak passes, $1\le k<K$, are all hacks and whose other responses all have reward $b$, and suppose every weak pass is audited. If caught hacks receive $c=2b-a$, the post-audit rewards are $r'_i=2b-r^w_i$, so the advantages are exactly negated, $A'_i=-A^w_i$, with centered advantages and with std normalization for any $\epsilon$ in the denominator. For a general $c<b$, let $\lambda=(b-c)/(a-b)$ and let $\sigma$ be the group's standard deviation before the audit; then $A'_i=-\lambda A^w_i$ with centered advantages and $A'_i=-\frac{\lambda(\sigma+\epsilon)}{\lambda\sigma+\epsilon}A^w_i$ with std normalization. If $c=b$, all rewards in the group equal $b$ and every advantage is zero.
\end{proposition}

So the reflected penalty is the only fixed correction that turns the push toward the hacks in such a group into an equal and opposite push under both advantage forms; with std normalization and $\epsilon=0$, every $c<b$ does so. In LeetCode, which uses std normalization, \LB{} with penalty $-0.25$ instead of $-2.5$ also suppressed hacking (one seed; \cref{app:prereg}). The strong-verdict reward only removes the push, and a partially discounted reward $b<c<a$, as in detector-based reward shaping, leaves the hacks with positive advantages \citep{hackfree2026}. Whenever $2b-a\le0$, as in both of our environments, a caught hack has the lowest reward in any group and hence a non-positive advantage. In a mixed group, the group-adaptive correction $2\mu_N-a$, with $\mu_N$ the mean reward of the other responses, negates every hack's advantage and keeps the other contrasts, but needs the strong labels of the whole group (\cref{prop:adaptive}).

\subsection{The value of a label and low-pass budgeted audits}

Consider whole-group designs at a frozen state. A design audits each group with $k$ weak passes with some probability, at an expected cost of $k$ labels per audited group. The objective is the expected advantage mass on hacks, $\E[J]$ with $J=\sum_{\text{groups}}\tfrac1K\sum_{i\in H}A_i$, which an audit design should make as negative as possible. Write $\mu_k^0$ and $\mu_k^1$ for the expected value of $\frac1K\sum_{i\in H}A_i$ in a $k$-pass group when it is not audited and when it is fully audited, and define the value per label
\begin{equation}
d_k=\frac{\mu_k^0-\mu_k^1}{k}.
\end{equation}

\begin{theorem}[Conditional optimality of \LB{}]
\label{thm:lb}
Fix the pass counts of the groups in a batch, suppose each group's composition is random with a distribution that depends only on its pass count and is independent of the design's randomness, and let $B$ be the expected label budget ($B=pW$). Consider whole-group designs that give every caught hack the same correction value, whose audit decisions depend only on the pass counts and on independent randomness, with expected cost at most $B$. If $d_k\ge0$ for all $k$ and $d_k$ is non-increasing in $k$, then \LB{} minimizes $\E[J]$ among these designs. If instead $|d_k-\tilde d_k|\le\gamma$ for some non-negative, non-increasing sequence $\tilde d$, \LB{} is within $2B\gamma$ of the optimum.
\end{theorem}

By linearity, $\E[J]$ falls by $d_k$ for each label spent on a $k$-pass group, so the problem is a fractional knapsack, and \LB{} spends labels in decreasing order of $d_k$. Any fixed correction $c\le a$ gives $d_k\ge0$ (\cref{prop:monotone}). The theorem compares whole-group designs that select by pass count alone and give caught hacks a fixed correction. Designs that read the prompts or the responses, audit part of a group, or rescale an audited group's correction fall outside this class; the group-level advantage correction of \citet{sakib2026audit} rescales and matches a full audit in expectation at any positive budget (\cref{prop:advipw}). With centered advantages and exchangeable compositions, the conclusion extends to all designs that use only the pass counts, including per-item and partial audits (\cref{thm:centered-lb}); with std normalization it does not (\cref{app:unit}). The monotonicity condition holds in a simple model.

\begin{proposition}[Pure groups]
\label{prop:pure}
Suppose that, for every $k$, the weak passes of a $k$-pass group are all hacks with probability $\rho_k$ and all correct otherwise, other responses have reward $b$, advantages use the unbiased group std with $\epsilon=0$, and caught hacks receive $2b-a$. Then $d_K=0$ and, for $1\le k<K$,
\begin{equation}
d_k=2\rho_k\sqrt{\frac{(K-1)(K-k)}{K^3\,k}},
\end{equation}
which is non-increasing in $k$ whenever $\rho_k$ is.
\end{proposition}

The factor $\sqrt{(K-k)/k}$ is proportional to the advantage of a weak pass among $k$ passes (\cref{fig:theory}a): a hack that shares its group with few other passes is reinforced the most, and a label that reverses it buys the most. On 72 frozen-state panels from both environments, where groups mix hacks and correct answers and, with malformed responses, a pass's advantage need not fall with $k$, \LB{} attained the optimum over the designs of \cref{thm:lb} in 123 of the 144 cases (each panel scored by two objectives) and came within $4.2\times10^{-4}$ of it in the rest (\cref{app:panels}). Designs that audit only part of a group can do better on these objectives, but such a design did worse than \LB{} in real training (\cref{sec:boundaries}).

\subsection{Unbiased per-item corrections under group std normalization}
\label{sec:theory-ipw}

Without std normalization, \cref{lem:marginals} makes the raw gradient of reward-level per-item \IPW{} equal in expectation to that of a full audit. With std normalization it is not, because the reward correction $(c-a)/p$ is large and lands on a single response before normalization.

\begin{proposition}[Saturation]
\label{prop:ipwstd}
Fix a group in which exactly one response receives reward $R$ and the other $K-1$ rewards lie in $[0,a]$. As $R\to-\infty$, the std-normalized advantage of that response tends to $-(K-1)/\sqrt K$ and every other advantage tends to $+1/\sqrt K$.
\end{proposition}

A normalized advantage never falls below $-(K-1)/\sqrt K$ (\cref{lem:samuelson}), so under \IPW{} with small $p$ a caught hack receives a bounded negative advantage, whatever the size of the correction, while every other response in its group, including uncaught hacks, receives nearly the same positive advantage. In a Countdown-Code group with four hacks and four failures at our budget, the expected hack mass under \IPW{} is $+0.016$, against $-0.468$ under a full audit and $+0.468$ without one, and as $p\to0$ it tends to the unaudited value (\cref{prop:ipw-small}). Rescaling the post-normalization advantage change of a whole audited group by $1/q_g$, by contrast, stays unbiased (\cref{prop:advipw}).

\section{Experimental setup}
\label{sec:experiments}

\paragraph{LeetCode (Qwen3-4B, LoRA).} We use the training code and data of \citet{wong2026monitor}, built on verl \citep{sheng2024hybridflow}. Each prompt is a programming problem; the weak verifier runs either the problem's tests or a \texttt{run\_tests} function supplied by the response, and a \emph{hack} passes the weak verifier but fails the problem's real test suite, which serves as the strong verifier. We keep their batch shape and model (16 prompts with 16 responses per step, LoRA of rank 32 on Qwen3-4B \citep{qwen2025qwen3}, group std normalization) and use a constant learning rate of $1.75\times10^{-5}$, because pilot arms with a larger rate became unstable within 40 steps (\cref{tab:config}). All arms start from the step-80 adapter of an unaudited run, where \nm{lc.start.h}\% of training rollouts are hacks, and train for 40 steps with a fresh optimizer state. The budget is $p=0.33$ of the weak passes and the reflected penalty is $2b-a=-2.5$.

\paragraph{Countdown-Code (Qwen2.5-3B, full fine-tuning).} In Countdown-Code \citep{khalifa2026countdown} the model writes a solution file with an arithmetic expression and the target, and a test file; the weak reward is 1.2 if the test passes, 0.2 for well-formed output and 0 otherwise. The dominant hack replaces the target in the solution file by the value of an expression that uses the given numbers; the strong verifier checks the expression literal against the original numbers and target. We train Qwen2.5-3B-Instruct \citep{qwen2025qwen25} with full-parameter GRPO in TRL \citep{vonwerra2020trl}, matching the original settings (32 prompts with 8 responses per step, 2048-token limit, no KL term; \cref{tab:config}). All arms start from step 32 of an unaudited run (\nm{cc.start.h}\% hacks at the first step) and train for 40 steps. The budget is $p=0.28$ and the reflected penalty is $2b-a=-0.8$.

\paragraph{Protocol.} Each comparison's arms, primary metric and decision rule were committed before its runs started. Arms of a comparison share the seed and therefore the prompts at every step, and we report paired bootstrap 95\% intervals over (step, prompt) units. The two main comparisons (\LB{} against \RB{} in LeetCode and against \IPW{} in Countdown-Code) have three seeds each, with a seed-level $t$-test; so do per-item random audits and the full-audit penalty comparison in LeetCode and \LB{} without std normalization in Countdown-Code. The remaining comparisons use one seed. Final checkpoints are evaluated on 1000 Countdown-Code test problems with 8 samples each and on the 119 medium and hard LeetCode test problems of \citet{wong2026monitor} with 32 samples each. Further baselines (\PAC{}, an oracle selection rule and audited-only learning) are described in \cref{app:details}.

\section{Predicting training from the frozen push}
\label{sec:results}

The analysis of \cref{sec:theory} characterizes the expected update at a frozen policy state. This section tests how well that frozen push, the hack mass $J$, predicts training: which selections and correction values suppress hacking (\cref{sec:res-select,sec:res-penalty}), how the outcome shifts with the budget, the penalty size and the audit unit (\cref{sec:res-budget}), and how often the sign of $J$ at the starting state predicted a run and where it broke down (\cref{sec:boundaries}).

Differences are first arm minus second, in percentage points, with paired bootstrap 95\% intervals over (step, prompt) units in training and over problems on held-out panels. These intervals capture variation within a pair of runs; variation across seeds enters through the three-seed comparisons and the seed-level tests. \Cref{fig:forest} summarizes the main paired comparisons, per-seed tables are in \cref{app:tables}, and seed-level tests in \cref{app:stats}.

\begin{figure}[t]
\centering
\includegraphics[width=\linewidth]{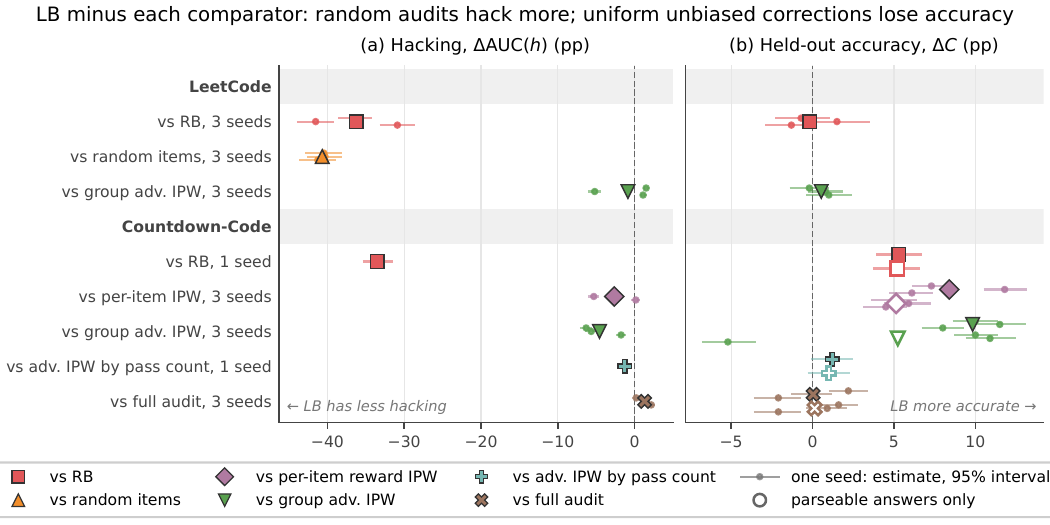}
\caption{Paired differences, \LB{} minus each comparator, with the same penalty and expected budget; the full audit checks every weak pass. Large markers are means over seeds; small dots and thin lines are each seed's estimate and paired bootstrap 95\% interval. Hollow markers count only parseable held-out answers. Held-out accuracy uses 119 LeetCode test problems with 32 samples each (against \RB{}), 353 held-out LeetCode training-split problems with 16 samples each (against the group-level correction), and 1000 Countdown-Code test problems with 8 samples each.}
\label{fig:forest}
\end{figure}

\subsection{Under the reflected penalty, which groups are audited decides the outcome}
\label{sec:res-select}

\LB{} and \RB{} use the same penalty and audit whole groups until a third of the weak passes are checked; only the order of the groups differs. In LeetCode, \LB{} removes hacking within about ten steps in all three seeds, while \RB{} lets it rise from about a quarter of rollouts to \nm{lc.RB.s1.hlate}--\nm{lc.RB.s2.hlate}\% in the last ten steps, close to the unaudited run (\nm{lc.NONE.hlate}\%; \cref{fig:main}b). Auditing single weak passes at random with the same penalty and budget fails as well, with late hacking of \nm{lc.IB.s1.hlate}--\nm{lc.IB.s3.hlate}\%. The area under the hacking curve is \nm{lc.LBRB.max.d_auc.abs}--\nm{lc.LBRB.min.d_auc.abs} points lower for \LB{} than for \RB{}, and \nm{lc.IBLB.min.d_auc.abs}--\nm{lc.IBLB.max.d_auc.abs} points lower than for random items (\cref{fig:forest}a). Late correctness shows no detectable difference between \LB{} and \RB{} (\cref{fig:lc} in \cref{app:tables}). Because the budget is a fraction of the weak passes, \LB{}'s spending falls as hacking is suppressed: it used \nm{lc.LBRB.mean.labelratio}\% of \RB{}'s labels and \nm{lc.LBIB.mean.labelratio}\% of those of random items (\cref{fig:main}c).

The difference starts with what each label finds. In the first five steps, when the arms are in similar states, each \LB{} label caught a hack \nm{lc.LB.s1.hit5}\%, \nm{lc.LB.s2.hit5}\% and \nm{lc.LB.s3.hit5}\% of the time, against \nm{lc.RB.s1.hit5}\%, \nm{lc.RB.s2.hit5}\% and \nm{lc.RB.s3.hit5}\% for \RB{}. On held-out problems, \RB{}'s hacking generalizes in all three seeds, at \nm{ho3.lc.s1.RB.H}\%, \nm{ho3.lc.s2.RB.H}\% and \nm{ho3.lc.s3.RB.H}\% of responses against \nm{ho3.lc.s1.LB.H}\%, \nm{ho3.lc.s2.LB.H}\% and \nm{ho3.lc.s3.LB.H}\% for \LB{} (\nm{ho.lc.start.H}\% at the starting checkpoint), while held-out accuracy shows no detectable difference (\cref{fig:forest}b).

In Countdown-Code, one seed of random whole-group audits with the same penalty and budget rule also failed to suppress hacking: late hacking was \nm{cc.RB.hlate}\% against \nm{cc.LB.s1.hlate}\% for \LB{}, and the area under the hacking curve was \nm{cc.RBLB.dauc.abs} points higher \ci{cc.RBLB.dauc}. Because hacking spread, it used \nm{cc.RBLB.labelratio.abs}\% more labels than \LB{}. On held-out problems its hacking persisted (\nm{ho.cc.RB.H}\% of responses) and \LB{} answered \nm{ho.cc.LBRB.dC.abs} points more problems correctly \ci{ho.cc.LBRB.dC}.

Strict hacking understates the difference in behavior: a LeetCode response can pass the real tests while defining its own \texttt{run\_tests}, and such responses carry the behavior that produces hacks \citep{wong2026monitor}. In the last ten steps, \nm{att.LB.s2.late.attempt}--\nm{att.LB.s1.late.attempt}\% of \LB{}'s responses define their own tests, against \nm{att.RB.s3.late.attempt}--\nm{att.RB.s2.late.attempt}\% for \RB{} (\cref{app:attempts}).

\subsection{The correction value decides whether targeting helps}
\label{sec:res-penalty}

Holding the selection fixed (\cref{tab:single}), auditing the 8 lowest-pass groups per step with the strong-verdict reward $c=b$ leaves late hacking at \nm{lc.LOW.hlate}\%, no better than auditing 8 random groups (\nm{lc.POS.hlate}\%), whereas with the reflected penalty the same 8 groups remove hacking (\nm{lc.LOWm.hlate}\%), with \nm{lc.LOWm_LOW.dClate.abs} points higher late correctness and \nm{lc.LOWm_LOW.labelratio}\% of the labels; they also beat a single-seed pilot that audited every weak pass with $c=b$ (\nm{lc.FULL.hlate}\% late hacking) at \nm{lc.LOWm_FULL.labelratio}\% of its labels. With $c=b$, a fully caught all-hack group contributes no update, while hacks in the many unaudited high-pass groups are still reinforced (\cref{prop:reflect}). In three later seeds of full coverage, $c=b$ left late strict hacking at \nm{att.FULLb.s1.late.H}\%, \nm{att.FULLb.s2.late.H}\% and \nm{att.FULLb.s3.late.H}\% and late attempts at \nm{att.FULLb.s1.late.attempt}\%, \nm{att.FULLb.s2.late.attempt}\% and \nm{att.FULLb.s3.late.attempt}\%, whereas the reflected penalty removed strict hacking and drove attempts to at most \nm{att.FULLr.s1.late.attempt}\% (\cref{app:attempts}).

\subsection{Budget, penalty size and audit unit}
\label{sec:res-budget}

\paragraph{Low budgets.} At a budget of 0.1 of the weak passes, \LB{} did not suppress hacking in LeetCode (late hacking \nm{lc.LBp10.s1.hlate}\% and \nm{lc.LBp10.s2.hlate}\% in two seeds), while at 0.2 it did (\nm{lc.LBp20.s1.hlate}\% and \nm{lc.LBp20.s2.hlate}\%) with \nm{lc.LBp20.min.labelratio}--\nm{lc.LBp20.max.labelratio}\% of the labels of \LB{} at 0.33. At 0.1 the audited groups are too few for their corrections to outweigh the reinforcement of the unaudited hacks. Multiplying the change in the audited groups' normalized advantages by 3.1 (the frozen-panel ratio of the full audit's change in hack mass to \LB{}'s; a design outside \cref{thm:lb}) suppressed hacking (\nm{lc.LBscale10.hlate}\%), with no detectable difference from \LB{} at 0.33 in the area under the hacking curve (\pmci{lc.LBscale10LBm.dauc}) and \nm{lc.LBscale10.labelratio}\% of its labels. The group-level advantage correction at 0.1 also suppressed hacking (\nm{lc.advIPWp10.hlate}\%), slightly more slowly (one seed each). A heavier penalty ($-11.5$) at 0.1 suppressed hacking only for a while: from \nm{lc.LBp10h.s1.minbin}\% around steps 16--20 it rose to \nm{lc.LBp10h.s1.lastbin}\% in steps 36--40, and to \nm{lc.LBp10hc40.bin2}\% within ten steps of a continuation with a fresh optimizer.

\paragraph{The audit unit under std normalization.} \Cref{lem:marginals} assumes centered advantages. With std normalization, LeetCode runs at equal marginals went both ways: with penalty 0 whole groups had slightly lower hacking than independent items (three seeds), and with a heavy penalty of $-11.5$ items suppressed hacking and whole groups did not (two seeds). An exact condition for one group explains both (\cref{app:unit}).

\subsection{Predictions from the starting state, and where the push stops predicting}
\label{sec:boundaries}

\paragraph{Predictions from the starting state.} For each run, we computed the hack mass $J$ of its design on a frozen panel of about 1000 prompts at its starting checkpoint and predicted from its sign whether hacking would be suppressed (late hacking below 2\%) or persist (10\% or more). Over the 39 runs whose designs we scored this way, 33 predictions held. For 16 runs the prediction was committed before the run started (\cref{app:prereg}), and 13 of these held, in both directions: \LB{} failed at a budget of 0.1 and succeeded at 0.2 in two seeds each, \LB{} with corrections scaled by 3.1 and the group-level advantage correction succeeded at 0.1, and random group audits failed in Countdown-Code. The other three ended between the two thresholds and nearer the predicted side: \nm{cc.advIPW.hlate}\% and \nm{lc.advIPW.s2.hlate}\% late hacking for two runs predicted to suppress hacking (seed 1 of the Countdown-Code group-level correction and seed 2 of the LeetCode one), and \nm{lc.LBp10h.s1.hlate}\% for one predicted not to (\LB{} with penalty $-11.5$ at 0.1, which suppressed hacking for about 15 steps before it rose); a continuation of that run, preregistered with separate thresholds, rose to \nm{lc.LBp10hc40.hlate}\% against a frozen-sign prediction of below 10\%. Of the other 23 runs, the sign matched all 13 LeetCode runs of the main comparison and its references (\LB{}, random group audits and random item audits in three seeds each, no audit, and the full audit in three seeds) and all six Countdown-Code runs of \LB{} and the full audit, but it missed three of the four audit-unit runs at penalty $-11.5$ and $p=0.4$ (\cref{app:unit}): whole-group audits ($J=-0.002$) did not suppress hacking in either seed, and per-item audits ($J=-0.067$) did in one seed but rose to 13.7\% in the other. Designs we did not score on the starting panels (the strong-verdict arms, the audit-unit runs at penalty 0, the per-item \IPW{} variants and \LB{} without std normalization) are not included. Within these two environments and related designs, the sign of the frozen push anticipated the outcome in most runs.

\paragraph{Light penalties and widespread hacking.} \LB{}'s ordering assumes that low-pass groups hold a large share of the hacks. In the LeetCode runs with $c=b$, hacking spread, and the share of hacks in groups with at least 12 of 16 weak passes rose from about half to over 80\%; the 8 lowest-pass groups then held no more hacks than random groups. Under the reflected penalty hacking collapsed within the first ten steps, before it could spread to high-pass groups.

\paragraph{An online rule that optimizes the frozen objective.} Auditing part of each group, with the number of labels chosen from the estimated composition of groups with each pass count, improves $J$ beyond \LB{} on frozen panels. In a preregistered LeetCode pilot from step 120 ($p=0.2$), such a rule was worse than \LB{}: its area under the hacking curve was \nm{lc.ON120_LB120.dauc.abs} points higher and its late hacking \nm{lc.ON120.hlate}\% against \nm{lc.LB120.hlate}\%, although its realized $J$ was more negative in every phase. So $J$ is not the first-order change in the hacking rate, which also depends on the directions of all responses' gradients and on the optimizer. A one-step diagnosis did not identify the cause, but showed that which of several similar hacks in a partially audited group receives the penalty is a large source of noise in a single update, which whole-group audits do not have (\cref{app:diag}).

\section{Beyond the expected update: unbiased corrections}
\label{sec:res-ipw}

Unbiased corrections target the full audit's update: on a frozen batch, per-item reward-level \IPW{} without std normalization has the full audit's expected update by \cref{lem:marginals}, and so does the group-level advantage correction of \citet{sakib2026audit} by \cref{prop:advipw}. Comparing the group-level correction with \LB{} and with a full audit therefore separates the expected update from the randomness of its estimate and the allocation of audits.

\begin{figure}[t]
\centering
\includegraphics[width=\linewidth]{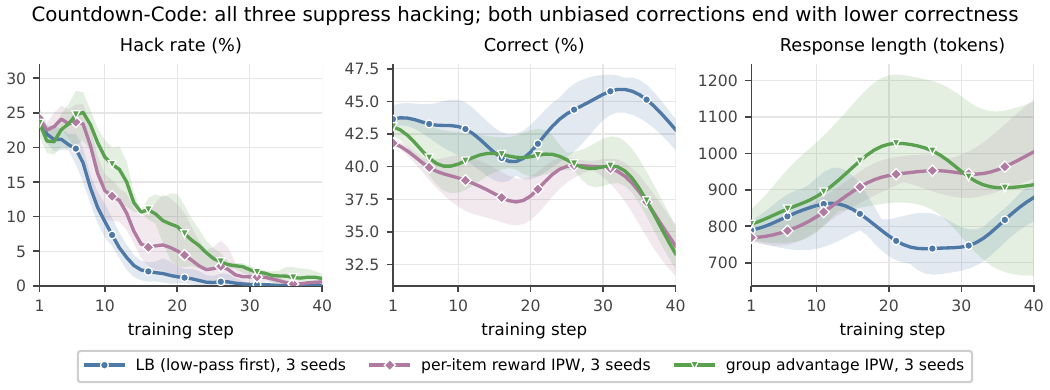}
\caption{Countdown-Code training curves with the reflected target ($-0.8$) at $p=0.28$: \LB{}, per-item reward-level \IPW{} and the group-level advantage correction of \citet{sakib2026audit}, each the mean over three seeds with the range shaded, smoothed with a Gaussian kernel of one step (hack rate) or three steps (correctness, length).}
\label{fig:cc}
\end{figure}

Per-item reward-level \IPW{} targets the same reflected reward as \LB{} in Countdown-Code, spends about the same number of labels, and, following \cref{lem:marginals}, runs without std normalization. Both suppress hacking, with late hacking below 1\% in every run, but \LB{} ends with higher late correctness in all three seeds, by \nm{cc.LBIPW.min.dClate.abs}--\nm{cc.LBIPW.max.dClate.abs} points, and answers more held-out problems correctly, by \nm{ho3.cc.LBIPW.min.C.abs}--\nm{ho3.cc.LBIPW.max.C.abs} points (\cref{fig:cc,fig:forest}). Part of the held-out gap comes from unfinished answers: \IPW{} left \nm{ho3.cc.IPW.trunc.min}--\nm{ho3.cc.IPW.trunc.max}\% of its answers cut at the 2048-token limit, against \nm{ho3.cc.LB.trunc.min}--\nm{ho3.cc.LB.trunc.max}\% for \LB{}. Among parseable answers \LB{} was still more accurate in every seed, by \nm{ho3.cc.LBIPW.min.Cparse.abs}--\nm{ho3.cc.LBIPW.max.Cparse.abs} points, with each interval above zero; comparisons among parseable answers condition on an answer being finished, so the subsets compared need not have the same difficulty. Continuing every cut answer to 4096 tokens from its own prefix made only \nm{ho4k.cc.IPW.rescued} of \nm{ho4k.cc.IPW.continued} correct, and \LB{}'s lead was unchanged at \nm{ho4k.cc.LBIPW.min.dC}--\nm{ho4k.cc.LBIPW.max.dC} points. The correction value matters for \IPW{} too: targeting $c=b$ instead of $-0.8$ leaves late hacking at \nm{cc.IPWtrue.hlate}\%, and adding group std normalization, which breaks its unbiasedness (\cref{sec:theory-ipw}), leaves it at \nm{cc.IPWstd.hlate}\% (single seeds; \cref{tab:single}).

\paragraph{Group-level advantage correction.} \citet{sakib2026audit} correct the advantages instead of the rewards: an audited group's normalized advantages move from their weak-reward values toward their full-audit values by the change divided by the audit probability, which matches a full audit in expectation (\cref{prop:advipw}). In LeetCode, auditing each group with probability 0.33 (the expected budget of \LB{}) suppressed hacking in two of three seeds, slightly faster than \LB{}, and left late hacking at \nm{lc.advIPW.s2.hlate}\% in the third; held-out accuracy showed no detectable difference from \LB{} in any seed (\cref{fig:forest}), although in LeetCode held-out accuracy barely moved even under widespread hacking (\cref{sec:res-select}), so this comparison has little power. In Countdown-Code at $p=0.28$ it also brought hacking down (late hacking \nm{cc.advIPW.min.hlate}--\nm{cc.advIPW.max.hlate}\%, the highest just above our 2\% threshold for suppression), but in all three seeds it ended with lower late training correctness, by \nm{cc.advIPWLB.min.dClate.abs}--\nm{cc.advIPWLB.max.dClate.abs} points, and \LB{} answered more held-out problems correctly, by \nm{ho.cc.LBadvIPW.min.dC}--\nm{ho.cc.LBadvIPW.max.dC} points, each interval above zero (\cref{fig:cc,fig:forest}). In seed 1, \nm{ho.cc.advIPW.s1.M}\% of the correction's held-out answers were unparseable, almost all cut at the length limit, and among parseable answers it was more accurate than \LB{}, but continuing the cut answers to 4096 tokens made only \nm{ho4k.cc.advIPW.s1.rescued} of \nm{ho4k.cc.advIPW.s1.continued} correct and left the gap at \nm{ho4k.cc.LBadvIPW.s1.dC.abs} points \ci{ho4k.cc.LBadvIPW.s1.dC}; in seeds 2 and 3 few answers were unparseable (\nm{ho.cc.advIPW.s2.M}\% and \nm{ho.cc.advIPW.s3.M}\%), and \LB{} was more accurate among parseable answers as well, by \nm{ho.cc.LBadvIPW.min23.dC_parseable}--\nm{ho.cc.LBadvIPW.max23.dC_parseable} points. Its late response length varied widely across seeds (\nm{cc.advIPW.min.lenlate}--\nm{cc.advIPW.max.lenlate} tokens, against \nm{cc.LB.min.lenlate}--\nm{cc.LB.max.lenlate} for \LB{}; \cref{fig:cc}) while its accuracy loss did not, so length growth alone cannot explain the accuracy loss across the three seeds.

Two preregistered controls locate the gap. On a common frozen batch, every unbiased group-level design has the full audit's advantages, and hence its raw first-order policy gradient, as its expectation. A full audit with the reflected penalty stayed within about two points of \LB{} on held-out problems in all three seeds, with \LB{} minus the full audit at \nm{ho.cc.LBFULL.s1.dC} \ci{ho.cc.LBFULL.s1.dC}, \nm{ho.cc.LBFULL.s2.dC.abs} \ci{ho.cc.LBFULL.s2.dC} and \nm{ho.cc.LBFULL.s3.dC} \ci{ho.cc.LBFULL.s3.dC}, far from the uniform correction's gap, so the correction does not lose by aiming at the full-audit update. Allocating the same correction by pass count, with audit probabilities that fall with the number of weak passes (weights chosen to reduce the variance of the corrected advantages on a frozen panel of 1000 groups whose 5250 weak passes were all checked by the strong verifier, labels not counted in its training budget of \nm{cc.optIPW.labels}; floor 0.1), recovered most of the gap at the same expected budget: its held-out accuracy was \nm{ho.cc.optIPWadvIPW.dC.abs} points above the uniform correction's \ci{ho.cc.optIPWadvIPW.dC} and \nm{ho.cc.LBoptIPW.dC.abs} points below \LB{}'s \ci{ho.cc.LBoptIPW.dC} (one seed). The loss therefore comes not from the full-audit update that the correction estimates but from the randomness of the estimate and how the audits behind it are allocated; which part of that randomness hurts we have not identified. At this budget \LB{} accepts a bias in the update for low noise, and its held-out accuracy stayed within about two points of a full audit's with \nm{cc.LBFULL.min.labelratio}\% of its labels.

\paragraph{\LB{} without std normalization.} \LB{} does not depend on std normalization. In Countdown-Code over three seeds, \LB{} without it also suppressed hacking and ended with higher late correctness than \LB{} with it, by \nm{cc.LBnostdLB.min.dClate.abs}--\nm{cc.LBnostdLB.max.dClate.abs} points, at the cost of longer responses and more labels (\cref{app:std}).

\section{Related work}
\label{sec:related}

\paragraph{Reward hacking under verifiable rewards.} Checks that a policy can satisfy without solving the task are a recurring failure of RLVR. Countdown-Code isolates the phenomenon in a small code environment \citep{khalifa2026countdown}; ImpossibleBench measures the propensity of models to exploit tests \citep{zhong2025impossiblebench}; patches counted as solved on SWE-bench fail the repository's full suite \citep{wang2025solved}; and systematic verifier errors can delay or derail RLVR \citep{egashira2026delay}. Optimizing against a monitor of reasoning teaches the policy to hide its intent \citep{baker2025monitoring}, a training-time instance of reward-model overoptimization \citep{gao2023overoptimization}, and \citet{simhack2026} predict when detectable hacking will emerge.

\paragraph{Interventions during training.} \citet{wong2026monitor} benchmark monitor-based penalties and screening in the LeetCode environment we use and observe that correct responses that also attempt the hack keep the behavior alive (\cref{app:attempts}). HackFree-GRPO \citep{hackfree2026} discounts and filters trajectories scored by a frozen detector under centered advantages and shows that a discounted hack can keep a positive advantage; it scores every trajectory and has no strong-verifier budget to allocate. GRAFT \citep{ward2026graft} routes the gradients of flagged samples into parameters removed after training. \citet{moya2026verifier} show that partial audits suffice when the correction outweighs the pressure toward hacking and propose \PAC{}, a projected gradient correction we compare against in adapted form; their audits are independent per item.

\paragraph{Budgeted audits and labels.} AI Control spends a scarce budget of trusted audits on the most suspicious outputs at deployment \citep{greenblatt2024control}. CARE acquires ground-truth labels where the gap between pseudo-label and true-label advantages is largest \citep{wang2026care}; our oracle selection uses this criterion with true labels. Selecting the hardest examples for GRPO under an annotation budget works best \citep{hardexamples2025}; \LB{} applies a related preference for low-pass groups online. PAIR applies Horvitz--Thompson reweighting to adaptive rollout allocation and notes that dividing by a random group standard deviation biases ratio estimators \citep{pair2026}. \citet{sakib2026audit} quantify, without training, the variance of whole-group and per-attempt random audits with reward-level and advantage-level inverse-propensity corrections for coding agents (\cref{sec:intro}).

\paragraph{Group normalization and correlated errors.} Dr.~GRPO removes std normalization from GRPO \citep{liu2025drgrpo}. \citet{gupta2026caches} shows that changing only the correlation of rewards within a group can reverse the expected group-normalized update, and \citet{xin2026independent} measure correlated verifier errors within groups. STAR-GRPO down-weights unreliable ratings before normalization \citep{star2026grpo}, and noise-corrected GRPO corrects exogenous reward noise in expectation \citep{mansouri2025noise}. \Cref{lem:marginals} is the centered-advantage counterpart of these observations for audits, and \cref{prop:ipwstd} makes explicit how std normalization defeats unbiased per-item corrections.

\section{Conclusion and limitations}
\label{sec:conclusion}

\paragraph{Conclusion.} A small strong-verifier budget can suppress reward hacking in GRPO, provided it is spent where uncaught hacks gain the most and caught hacks are pushed away rather than merely zeroed. At a frozen state, auditing whole groups with the fewest weak passes first minimizes the advantage mass left on hacks among whole-group pass-count designs with a fixed correction, when a label's value does not increase with the pass count. From the starting state, the sign of this mass predicted whether hacking would be suppressed in 33 of 39 runs, including 13 of 16 preregistered ones. In our experiments the following recipe worked in both environments: give caught hacks the reflected reward $2b-a$; audit whole groups with the fewest weak passes first; set the budget as a fraction of weak passes, so that spending falls when weak passes fall as hacking is suppressed; and avoid per-item reward-level inverse-propensity corrections under group std normalization. Randomness matters as well. Unbiased corrections can match \LB{} on hacking, but in Countdown-Code their randomness cost correctness at the same budget: a full audit stayed within about two points of \LB{}, and reallocating the audits recovered most of the loss.

\paragraph{Limitations.} Our experiments cover two environments, 3B and 4B models, and 40-step windows that start after hacking has emerged; we have not tested whether hacking returns over longer training or at larger scale. In both environments the strong verifier is cheap enough that the budget is a design choice rather than a necessity; we have not tested a setting where the strong verifier is itself expensive or imperfect. We treat the strong verifier as reliable and independent of the policy, and a strong verifier with blind spots would push the policy toward them, as observed for optimized monitors \citep{baker2025monitoring,ward2026graft}. The two main comparisons and several ablations have three seeds; the group-level advantage correction has three seeds in each environment and the full audit three in Countdown-Code, and other comparisons, including the allocated correction, have one or two. The correctness advantage of \LB{} over per-item \IPW{} is positive in all three seeds, with $p=\nm{cc.LBIPW.p.dClate}$ in training and $p=\nm{ho3.cc.LBIPW.p.C}$ on held-out problems at the seed level, and against the group-level correction it is positive in all three seeds, among parseable answers in two of them. Our decision metrics, hacking and correctness on training rollouts, missed length drift, which appeared in rollout lengths and on held-out panels. Our analysis is first order at a frozen state: the sign of its objective predicted whether hacking would be suppressed in 33 of 39 runs (13 of 16 preregistered), but its magnitude does not rank designs, as the online rule showed. These predictions come from two environments and related designs and have not been compared with a simple baseline. The reflected penalty is exact only when non-passing responses share the reward $b$.

\bibliographystyle{plainnat}
\bibliography{refs,generated/arxiv}

\appendix
\section{Proofs}
\label{app:proofs}

Throughout, the batch (prompts, responses, weak rewards, classes and response lengths) is fixed and the audit design is the only source of randomness, except in \cref{thm:lb,prop:pure}, where compositions are random given pass counts. All random objects take finitely many values. Token-mean aggregation weights each response's score by quantities that depend on response lengths, not on rewards, so the weights are fixed given the batch and are absorbed into $\nabla_\theta\log\pi_\theta(y_i\mid x_i)$. Normalized advantages are set to zero when $\sigma=\epsilon=0$.

\paragraph{\Cref{lem:marginals}.} The post-audit reward is $r_i=r^w_i+m_i\mathbf 1[i\in H](c_i-a)$, where $c_i$ is fixed given the batch (for \IPW{} it depends on $\pi_i$ only). Hence $\E[r_i]=r^w_i+\mathbf 1[i\in H]\,\pi_i(c_i-a)$. With centered advantages, $A_i=r_i-\frac1K\sum_{j\in g(i)}r_j$ is linear in the rewards of its group $g(i)$, so $\E[A_i]=\E[r_i]-\frac1K\sum_{j\in g(i)}\E[r_j]$, and $\E[g]=\sum_i\E[A_i]\nabla_\theta\log\pi_\theta(y_i\mid x_i)$ because the score vectors are fixed. All of these depend on the design only through $\pi_i(c_i-a)$, $i\in H$; no independence between audit decisions is needed. Whole-group and per-item designs with the same marginals and the same fixed $c$ therefore agree. For \IPW{}, $\pi_i(c_i-a)=p\cdot(c-a)/p=c-a$, the value for a full audit ($\pi_i=1$) with correction $c$. \qed

\paragraph{\Cref{prop:reflect}.} Write $e_i=r^w_i-\bar r^w$; the weak rewards are $a$ on the $k$ hacks and $b$ elsewhere, so $e_i=(a-b)(K-k)/K$ for a hack and $-(a-b)k/K$ for a failure, and $\sigma>0$ because $1\le k<K$. After a full audit with correction $c$ the deviations are $(c-b)(K-k)/K$ and $-(c-b)k/K$, that is, $-\lambda e_i$ with $\lambda=(b-c)/(a-b)$. For $\lambda\ne0$ the standard deviation becomes $|\lambda|\sigma$, centered advantages scale by $-\lambda$, and normalized ones become $-\lambda e_i/(|\lambda|\sigma+\epsilon)$. For $c<b$ ($\lambda>0$) this is $-\frac{\lambda(\sigma+\epsilon)}{\lambda\sigma+\epsilon}A^w_i$, which equals $-A^w_i$ for every $\epsilon$ exactly when $\lambda=1$, that is, $c=2b-a$, where indeed $r'_i=2b-r^w_i$; with $\epsilon=0$ it equals $-A^w_i$ for every $c<b$. For $b<c<a$ ($-1<\lambda<0$) the advantages keep their signs, and for $c=b$ all rewards equal $b$ and every advantage is zero. If $2b-a\le0$, the reward $2b-a$ is the smallest possible reward, so a caught hack has a non-positive advantage in any group. \qed

\paragraph{\Cref{thm:lb}.} Since advantages are computed within groups, $J=\sum_g\Phi_g(z_g)$, where $\Phi_g$ is the contribution of group $g$ and $z_g\in\{0,1\}$ its audit decision. A design in the class audits group $g$ with a probability $x_g$ that does not depend on the group's composition, so $\E[\Phi_g(z_g)]=(1-x_g)\mu^0_{k_g}+x_g\mu^1_{k_g}=\mu^0_{k_g}-y_gd_{k_g}$ with $y_g=x_gk_g\in[0,k_g]$, the expected labels spent on group $g$; groups without weak passes contribute nothing. Hence $\E[J]=\sum_g\mu^0_{k_g}-V(y)$ with $V(y)=\sum_{g:\,k_g\ge1}y_gd_{k_g}$, and the problem is to maximize $V$ over $\mathcal Y=\{y:0\le y_g\le k_g,\ \sum_gy_g\le B\}$; every $y\in\mathcal Y$ is attained by independent audits with probabilities $y_g/k_g$. Condition on \LB{}'s tie-breaking and list the groups with $k_g\ge1$ in its order, $g_1,\dots,g_n$ with $k_{g_1}\le\dots\le k_{g_n}$; let $C_j=\sum_{l\le j}k_{g_l}$. \LB{} audits $g_1,g_2,\dots$ in full while they fit and the first group that does not fit with probability (remaining budget)$/k$, so by induction its allocation is $y^{\LB}_{g_j}=\min(k_{g_j},\max(0,B-C_{j-1}))$ and its partial sums are $\min(B,C_j)$; in particular its expected cost is $B$, and its decisions depend only on pass counts and tie-breaking. Put $\beta_j=d_{k_{g_j}}$, which is non-negative and non-increasing in $j$, $\beta_{n+1}=0$, and $Y_j=\sum_{l\le j}y_{g_l}$. Summation by parts gives $V(y)=\sum_{j=1}^nY_j(\beta_j-\beta_{j+1})$ with non-negative coefficients, so $V$ is non-decreasing in each $Y_j$; every $y\in\mathcal Y$ has $Y_j\le\min(B,C_j)$, the partial sums of \LB{}'s allocation, so $V(y^{\LB})\ge V(y)$. Tied groups share the same $d$, so this holds for every tie-breaking and on average. For the approximate case let $\tilde V(y)=\sum_gy_g\tilde d_{k_g}$; the same argument shows that $y^{\LB}$ maximizes $\tilde V$, and $|V(y)-\tilde V(y)|\le\gamma\sum_gy_g\le\gamma B$ for every $y\in\mathcal Y$. Hence $V(y)\le\tilde V(y)+\gamma B\le\tilde V(y^{\LB})+\gamma B\le V(y^{\LB})+2\gamma B$. \qed

\paragraph{\Cref{prop:pure}.} In a $k$-pass group with $k$ rewards $a$ and $K-k$ rewards $b$, $1\le k<K$, the mean is $b+(a-b)k/K$, a pass deviates from it by $(a-b)(K-k)/K$, and the unbiased variance is $(a-b)^2k(K-k)/(K(K-1))$, so with $\epsilon=0$ the normalized advantage of a pass is
\[
\alpha_k=\frac{(a-b)(K-k)/K}{(a-b)\sqrt{k(K-k)/(K(K-1))}}=\sqrt{\frac{(K-1)(K-k)}{Kk}} .
\]
An unaudited all-hack group contributes $k\alpha_k/K$ to $J$, and by \cref{prop:reflect} a fully audited one contributes the negative of this; groups whose passes are correct contribute zero either way. Hence $\mu^0_k-\mu^1_k=2\rho_kk\alpha_k/K$ and $d_k=2\rho_k\alpha_k/K$. For $k=K$ the group is all-hack or all-correct by assumption, every response has the same reward before and after the audit, and $d_K=0$. Both $\sqrt{(K-k)/k}$ and $\rho_k$ are non-negative and non-increasing in $k$, and so is their product. \qed

\paragraph{\Cref{prop:ipwstd}.} Let the response with reward $R$ be the first and $S=\sum_{j\ge2}r_j\in[0,(K-1)a]$. Then $\bar r=(R+S)/K$, $R-\bar r=\frac{K-1}{K}R+O(1)$ and $r_j-\bar r=-\frac RK+O(1)$ as $R\to-\infty$, so the unbiased variance satisfies
\[
\sigma^2=\frac{1}{K-1}\Big[\frac{(K-1)^2R^2}{K^2}+(K-1)\frac{R^2}{K^2}\Big](1+o(1))=\frac{R^2}{K}(1+o(1)),
\]
and $\sigma+\epsilon=\frac{|R|}{\sqrt K}(1+o(1))$ for fixed $\epsilon$. Hence $A_1\to-\frac{K-1}{\sqrt K}$ and $A_j\to\frac{1}{\sqrt K}$ for $j\ge2$. \qed

\section{Further theoretical results}
\label{app:theory}

These results support remarks in \cref{sec:theory}; the same setting and conventions apply.

\begin{lemma}[A bound on normalized advantages]
\label{lem:samuelson}
For every group reward vector $r$ with $\sigma(r)>0$, every $\epsilon\ge0$ and every $i$, the normalized advantage satisfies $|A_i|\le(K-1)/\sqrt K$.
\end{lemma}

\begin{proof}
Let $e_j=r_j-\bar r$, so that $\sum_je_j=0$ and $\sum_je_j^2=(K-1)\sigma^2$. By the Cauchy--Schwarz inequality over $j\ne i$, $e_i^2=\big(\sum_{j\ne i}e_j\big)^2\le(K-1)\big((K-1)\sigma^2-e_i^2\big)$, so $|e_i|\le\sigma(K-1)/\sqrt K$. Divide by $\sigma+\epsilon\ge\sigma$.
\end{proof}

\begin{proposition}[Group-adaptive reflection]
\label{prop:adaptive}
Let a group contain $h$ hacks, $1\le h\le K-1$, all audited, and let $\mu_N$ be the mean reward of its $N=K-h$ other responses. If every hack receives $2\mu_N-a$ and $\Delta=2h(a-\mu_N)/K$, then each hack's deviation from the group mean is negated, each other response's deviation increases by $\Delta$, and the group's standard deviation is unchanged. Hence every hack's advantage is negated, for both advantage forms, and the advantages of the other responses all shift by the same amount ($\Delta$ centered; $\Delta/(\sigma+\epsilon)$ normalized, or $0$ if $\sigma=\epsilon=0$), so their contrasts are unchanged.
\end{proposition}

\begin{proof}
The weak mean is $\bar r=(ha+N\mu_N)/K$, so $\bar r-\mu_N=\Delta/2$, and the new mean is $\bar r-\Delta$. A hack's new deviation is $(2\mu_N-a)-\bar r+\Delta=-(a-\bar r)$ because $2(\bar r-\mu_N)=\Delta$. The other responses keep their rewards while the mean falls by $\Delta$. With $e_j$ their old deviations, $\sum_j(e_j+\Delta)^2=\sum_je_j^2+2\Delta\sum_je_j+N\Delta^2$, and $\sum_je_j=N(\mu_N-\bar r)=-N\Delta/2$, so the sum of squared deviations is unchanged.
\end{proof}

This correction needs the strong labels of every weak pass of the group and can be computed only after all of them have been audited. When the other responses all have reward $b$, it equals $2b-a$ (\cref{prop:reflect}).

\begin{proposition}[Lowering the hack reward never raises the hack mass]
\label{prop:monotone}
Fix a group with $h$ hacks, $1\le h\le K-1$, whose hacks share a common reward $t$ while the other rewards are fixed. For centered advantages, and for normalized advantages with any $\epsilon\ge0$, the group's hack mass $\Phi(t)=\frac1K\sum_{i\in H}A_i$ is non-decreasing in $t$. Consequently a whole-group audit with any fixed correction $c\le a$ does not raise the hack mass of any group, and $d_k\ge0$ for every $k$.
\end{proposition}

\begin{proof}
With $N=K-h$, $\mu_N$ and $S_N$ the mean and sum of squared deviations of the other rewards, and $x=t-\mu_N$, each hack deviates from the mean by $Nx/K$ and the sum of squared deviations is $S_N+hNx^2/K$. The centered hack mass is $hNx/K^2$. The normalized one is $(hN/K^2)\,x/(\sqrt{\gamma'+\beta x^2}+\epsilon)$ with $\gamma'=S_N/(K-1)\ge0$ and $\beta=hN/(K(K-1))>0$, whose derivative in $x$ is $\big(\gamma'/\sqrt{\gamma'+\beta x^2}+\epsilon\big)/\big(\sqrt{\gamma'+\beta x^2}+\epsilon\big)^2\ge0$ where $\gamma'+\beta x^2>0$; if $\gamma'=0$ the function is $x/(\sqrt\beta|x|+\epsilon)$ with value $0$ at $x=0$, which is also non-decreasing. A whole-group audit lowers the common hack reward from $a$ to $c\le a$.
\end{proof}

\begin{remark}[Other group-separable objectives]
\label{rem:objectives}
The proof of \cref{thm:lb} uses only that the objective is a sum over groups of a quantity determined by each group's composition and audit decision. It therefore holds for any such objective, for example $J$ minus the advantage mass on correct responses, with $d_k$ defined from that objective.
\end{remark}

\begin{proposition}[Centered advantages]
\label{prop:centered-dk}
With centered advantages and a fixed correction $c<a$, a full audit lowers the hack mass of any group with $h$ hacks by $(a-c)h(K-h)/K^2$, whatever its other responses. Hence $d_k=(a-c)\,\E[h(K-h)\mid k]/(kK^2)\ge0$. If each pass of a $k$-pass group is a hack independently with probability $\eta\in(0,1]$, then $d_k=(a-c)\eta(K-1+\eta-k\eta)/K^2$, which strictly decreases in $k$, also for groups that mix hacks and correct answers.
\end{proposition}

\begin{proof}
A full audit lowers each hack's reward by $a-c$ and the group mean by $h(a-c)/K$, so each hack's centered advantage falls by $(a-c)(1-h/K)$; sum over the $h$ hacks and divide by $K$. In the binomial model $\E[h(K-h)]=Kk\eta-k\eta(1-\eta)-k^2\eta^2$.
\end{proof}

\begin{theorem}[\LB{} among all pass-count designs, centered advantages]
\label{thm:centered-lb}
Under the assumptions of \cref{thm:lb}, suppose in addition that, given the number of hacks in a group, every set of its passes of that size is equally likely to be the hacks. With centered advantages and a fixed correction $c<a$, consider any design with this correction (per-item, whole-group or partial) whose audit law depends only on the pass counts and on independent randomness, and let $y_g$ be its expected labels on group $g$. Then $\E[J]=\sum_g\mu^0_{k_g}-\sum_{g:\,k_g\ge1}y_gd_{k_g}$ with $d_k$ from \cref{prop:centered-dk}, so if $d_k$ is non-increasing, \LB{} minimizes $\E[J]$ over all such designs.
\end{theorem}

\begin{proof}
For a fixed composition with hack set $H_g$ of size $h$, catching hack $i$ (probability $\pi_i$) lowers the group's hack mass by $(a-c)(1-h/K)/K$ by linearity, so the expected hack mass is $\Phi_g(0)-(a-c)(1-h/K)K^{-1}\sum_{i\in H_g}\pi_i$. The marginals do not depend on the composition, and by exchangeability $\E[\sum_{i\in H_g}\pi_i\mid h]=(h/k_g)y_g$. Taking expectations gives $\E[\Phi_g]=\mu^0_{k_g}-y_gd_{k_g}$, which has the form used in the proof of \cref{thm:lb}.
\end{proof}

With std normalization the conclusion fails: in the group of \cref{app:unit}, per-item audits at the same marginals lower the hack mass more than a whole-group audit when the penalty is heavy.

\begin{proposition}[Per-item reward-level \IPW{} at small rates]
\label{prop:ipw-small}
Fix a group with $h\ge1$ hacks, normalized advantages with $\epsilon\ge0$, and independent per-item audits with probability $p$ in which a caught hack receives any reward (possibly depending on $p$). Let $\Phi_w$ be the group's hack mass without an audit and $C_{K,h}=\sqrt{h(K-h)(K-1)/K^3}$. Then $|\E[\Phi_p]-\Phi_w|\le2C_{K,h}(1-(1-p)^h)\le2C_{K,h}hp$, so the expected hack mass tends to its unaudited value as $p\to0$.
\end{proposition}

\begin{proof}
Since the normalized advantages sum to zero, $\Phi=\frac1K\sum_iA_i(\mathbf 1[i\in H]-h/K)$, and by the Cauchy--Schwarz inequality $|\Phi|\le\frac1K\sqrt{K-1}\sqrt{h(K-h)/K}=C_{K,h}$, using $\sum_iA_i^2\le K-1$. If no hack is caught, which has probability $(1-p)^h$, the rewards equal the weak rewards and $\Phi_p=\Phi_w$; otherwise $|\Phi_p-\Phi_w|\le2C_{K,h}$.
\end{proof}

\Cref{prop:ipwstd} does not imply that \IPW{} with std normalization fails in general: with $K=2$, one hack and one failure at $b$, a target $c<b$ and $\epsilon=0$, the hack's advantage is $+1/\sqrt2$ when it is not caught and $-1/\sqrt2$ when it is, so its expected advantage $(1-2p)/\sqrt2$ is negative for $p>1/2$.

\begin{proposition}[Group-level advantage correction]
\label{prop:advipw}
For each group let $Z_g\sim\mathrm{Bernoulli}(q_g)$ with $q_g\in(0,1]$ fixed given the batch, let $r^F$ be the rewards after a full audit with correction $c$, and set $\hat A_i=A_i(r^w)+(Z_g/q_g)(A_i(r^F)-A_i(r^w))$ for $i$ in group $g$, for either advantage form. Then $\E[\hat A_i]=A_i(r^F)$, so the expected update equals that of a full audit; $\mathrm{Var}(\hat A_i)=\frac{1-q_g}{q_g}(A_i(r^F)-A_i(r^w))^2$; the expected number of labels is $\sum_gq_gk_g$; and $\sum_{i\in g}\hat A_i=0$.
\end{proposition}

\begin{proof}
$\E[Z_g/q_g]=1$ and $\mathrm{Var}(Z_g/q_g)=(1-q_g)/q_g$, and $\hat A_i$ is affine in $Z_g/q_g$; both advantage vectors sum to zero within the group.
\end{proof}

In our implementation the advantages are normalized before the correction and the trainer's own normalization is switched off: normalizing $\hat A$ again would be nonlinear in $Z_g$ and would destroy unbiasedness, whereas the trainer's group centering leaves $\hat A$ unchanged. The design generally falls outside the class of \cref{thm:lb}, because the update applied to an audited group is the full-audit change multiplied by $1/q_g$.

\paragraph{A worked example of saturation.} At the Countdown-Code budget ($K=8$, $p=0.28$, so a caught hack receives $R=1.2+(-0.8-1.2)/0.28\approx-5.94$), a group with four hacks of which one is caught and four failures at $0.2$ gives the caught hack a normalized advantage of $-2.42$, the uncaught hacks $+0.59$ and the failures $+0.17$; the limits of \cref{prop:ipwstd} are $-2.47$ and $+0.35$. A whole-group audit of this group gives the hacks $-0.94$ and the failures $+0.94$ (\cref{prop:reflect}). Averaged over the binomial number of caught hacks, the group's expected hack mass is $+0.016$ under \IPW{}, $-0.468$ under a full audit and $+0.468$ without an audit; with centered advantages \IPW{} and the full audit agree at $-0.25$. This compares \IPW{} with the full audit, not with \LB{} at the same budget: auditing this one group as a whole with probability $0.28$ gives $+0.206$.

\paragraph{Frozen-panel check of the monotonicity condition.}
\label{app:panels}
We checked the condition of \cref{thm:lb} on frozen-state panels of about 1000 prompts each: 48 Countdown-Code combinations (steps 12--32 of an unaudited run, penalty $-0.8$ or $-3.8$, audit rate 0.1--0.4) and 24 LeetCode combinations (steps 70, 80 and 120; penalty $-2.5$ or $-11.5$; rate 0.1--0.4). Each was scored by $J$ alone and by $J$ minus the advantage mass on correct responses (\cref{rem:objectives}). With the panels' empirical label values, \LB{} attained the optimum over the designs of \cref{thm:lb} in 123 of the 144 cases and was within $4.2\times10^{-4}$ of it in the rest.

\section{Experimental details}
\label{app:details}

\begin{table}[h]
\centering
\small
\caption{Training configurations.}
\label{tab:config}
\begin{tabular}{lll}
\toprule
& LeetCode & Countdown-Code \\
\midrule
Model & Qwen3-4B, LoRA rank 32, $\alpha=32$ & Qwen2.5-3B-Instruct, all parameters \\
Framework & verl, code of \citet{wong2026monitor} & TRL 0.26.2 with vLLM \\
Prompts $\times$ responses per step & $16\times16$ & $32\times8$ \\
Optimizer updates per step & 1 & 2 \\
Learning rate (constant) & $1.75\times10^{-5}$ & $3\times10^{-6}$ \\
AdamW $(\beta_1,\beta_2)$, weight decay & $(0.9,0.99)$, $0.1$ & $(0.9,0.999)$, $0.01$ \\
Gradient clipping, clip ratio & 1.0, 0.2 & 1.0, 0.2 \\
KL to reference & $10^{-3}$ (low-variance estimator) & none \\
Advantage & group std, $\epsilon=10^{-6}$ & group std, $\epsilon=10^{-4}$ (or none) \\
Loss aggregation & token mean & token mean \\
Sampling & temperature 0.7, top-$p$ 0.95 & temperature 1.0 \\
Prompt / response limit (tokens) & 1536 / 1536 & 2048 / 2048 \\
Starting checkpoint & step-80 adapter of an unaudited run & step 32 of an unaudited run \\
Training steps & 40 & 40 \\
Audit rate $p$, reflected penalty & 0.33, $-2.5$ & 0.28, $-0.8$ \\
Weak rewards $(a,b)$ & $(3.5,0.5)$ & $(1.2,0.2)$ \\
\bottomrule
\end{tabular}
\end{table}

\paragraph{Strong verifiers.} In LeetCode, a response is a hack when it passes the weak verifier only through its own \texttt{run\_tests} function and fails the problem's tests; responses that pass the problem's tests are correct even if they also define \texttt{run\_tests}. In Countdown-Code, a weak pass is a hack when its expression literal does not use exactly the given numbers or does not evaluate to the given target.

\paragraph{Baselines.} In Countdown-Code we also compare against three alternatives. \PAC{} \citep{moya2026verifier} keeps the weak reward and adds a correction along an inverse-propensity estimate of the hack-probability gradient, projected to be orthogonal to the reward gradient. The original Euclidean rule collapsed the policy within one step under full-parameter training, so we use an adaptation in which the projection is taken in the metric of AdamW's diagonal preconditioner after the batch's optimizer updates; the guarantees of the original method, stated for a fixed metric and gradient flow, do not carry over. The correction size was selected on a separate development state. An oracle selection rule (\textsc{Cag}) reads the semantic labels and audits groups in decreasing order of the corrective advantage gap, the acquisition criterion of CARE \citep{wang2026care}, under \LB{}'s budget; it is not the full CARE method. Audited-only learning (\textsc{Ao}) uses \LB{}'s audits but drops unaudited groups from the update.

\paragraph{Held-out panels and protocol.} Held-out evaluation uses the final checkpoints of all three seeds of the two main comparisons, seeds 2 and 3 of no-std \LB{}, and the single-seed Countdown-Code arms: 1000 Countdown-Code test problems with 8 samples each, and the 119 medium and hard LeetCode test problems of \citet{wong2026monitor} with 32 samples each, plus 353 held-out training-split LeetCode problems with 16 samples each for the difficulty strata of \cref{app:attempts} and the group-level correction. Paired bootstrap intervals use 4000 resamples over 640 (LeetCode) or 1280 (Countdown-Code) (step, prompt) units, or over problems.
\paragraph{Compute.} Training ran on two NVIDIA RTX PRO 6000 GPUs, one arm per GPU, at about two hours per 40-step arm. Held-out panels ran on separate RTX 3090 and A6000 GPUs.

\section{Per-seed results}
\label{app:tables}

\begin{figure}[h]
\centering
\includegraphics[width=\linewidth]{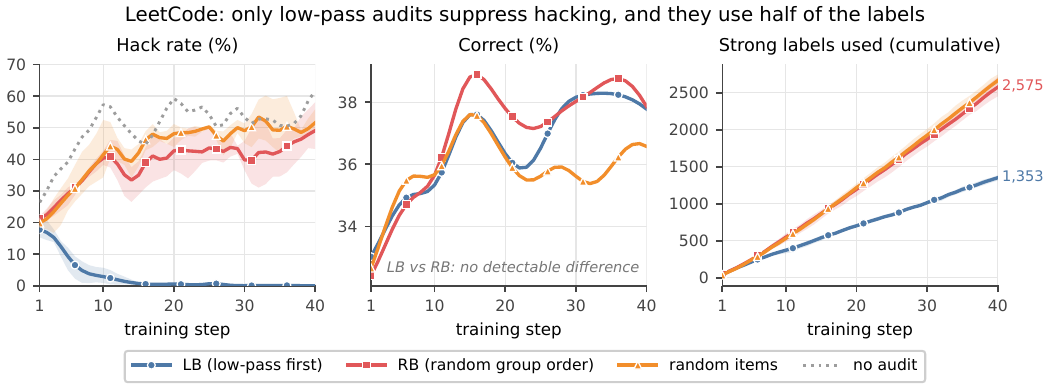}
\caption{LeetCode training curves with the reflected penalty ($-2.5$) and the same budget rule ($p=0.33$): mean over three seeds with their range shaded (correctness: means only), smoothed with a Gaussian kernel of one step (hack rate) or three steps (correctness).}
\label{fig:lc}
\end{figure}

In the tables, entries are estimates with the distances to the upper and lower ends of the 95\% interval, $x^{+u}_{-l}$.

\begin{table}[h]
\centering
\footnotesize
\caption{LeetCode: \LB{} minus \RB{} (reflected penalty $-2.5$, $p=0.33$), per seed. Paired bootstrap 95\% intervals over (step, prompt) units. Labels are strong-verifier calls over 40 steps.}
\label{tab:lbrb}
\setlength{\tabcolsep}{5pt}
\begin{tabular}{lccccc}
\toprule
Seed & $\Delta\mathrm{AUC}(h)$ & $\Delta h$, steps 31--40 & $\Delta C$, steps 31--40 & $\Delta$pass@16 & Labels \LB{} / \RB{} \\
\midrule
1 & \pmci{lc.LBRB.s1.dauc} & \pmci{lc.LBRB.s1.dhlate} & \pmci{lc.LBRB.s1.dClate} & \nm{lc.LBRB.s1.dpass16} & \nm{lc.LB.s1.labels} / \nm{lc.RB.s1.labels} \\
2 & \pmci{lc.LBRB.s2.dauc} & \pmci{lc.LBRB.s2.dhlate} & \pmci{lc.LBRB.s2.dClate} & \nm{lc.LBRB.s2.dpass16} & \nm{lc.LB.s2.labels} / \nm{lc.RB.s2.labels} \\
3 & \pmci{lc.LBRB.s3.dauc} & \pmci{lc.LBRB.s3.dhlate} & \pmci{lc.LBRB.s3.dClate} & \nm{lc.LBRB.s3.dpass16} & \nm{lc.LB.s3.labels} / \nm{lc.RB.s3.labels} \\
\midrule
Mean ($p$) & \nm{lc.LBRB.mean.d_auc} (\nm{lc.LBRB.p.d_auc}) & \nm{lc.LBRB.mean.d_h} (\nm{lc.LBRB.p.d_h}) & \nm{lc.LBRB.mean.d_c} (\nm{lc.LBRB.p.d_c}) & \nm{lc.LBRB.mean.d_p16} (\nm{lc.LBRB.p.d_p16}) & \nm{lc.LBRB.mean.labelratio}\% \\
\bottomrule
\end{tabular}
\end{table}

\begin{table}[h]
\centering
\footnotesize
\caption{Countdown-Code: \LB{} minus per-item \IPW{} (target $-0.8$, $p=0.28$), per seed. \LB{} uses group std normalization, \IPW{} does not.}
\label{tab:lbipw}
\setlength{\tabcolsep}{5pt}
\begin{tabular}{lccccc}
\toprule
Seed & $\Delta\mathrm{AUC}(h)$ & $\Delta h$, steps 31--40 & $\Delta C$, steps 31--40 & $\Delta$pass@8 & Labels \LB{} / \IPW{} \\
\midrule
1 & \pmci{cc.LBIPW.s1.dauc} & \pmci{cc.LBIPW.s1.dhlate} & \pmci{cc.LBIPW.s1.dClate} & \nm{cc.LBIPW.s1.dpass8} & \nm{cc.LB.s1.labels} / \nm{cc.IPW.s1.labels} \\
2 & \pmci{cc.LBIPW.s2.dauc} & \pmci{cc.LBIPW.s2.dhlate} & \pmci{cc.LBIPW.s2.dClate} & \nm{cc.LBIPW.s2.dpass8} & \nm{cc.LB.s2.labels} / \nm{cc.IPW.s2.labels} \\
3 & \pmci{cc.LBIPW.s3.dauc} & \pmci{cc.LBIPW.s3.dhlate} & \pmci{cc.LBIPW.s3.dClate} & \nm{cc.LBIPW.s3.dpass8} & \nm{cc.LB.s3.labels} / \nm{cc.IPW.s3.labels} \\
\midrule
Mean ($p$) & \nm{cc.LBIPW.mean.dauc} (\nm{cc.LBIPW.p.dauc}) & \nm{cc.LBIPW.mean.dhlate} (\nm{cc.LBIPW.p.dhlate}) & \nm{cc.LBIPW.mean.dClate} (\nm{cc.LBIPW.p.dClate}) & \nm{cc.LBIPW.mean.dpass8} (\nm{cc.LBIPW.p.dpass8}) & $+$\nm{cc.LBIPW.labels.excess}\% \\
\bottomrule
\end{tabular}
\end{table}

\begin{table}[h]
\centering
\footnotesize
\caption{Held-out test problems, per seed: \LB{} minus \RB{} in LeetCode (119 problems $\times$ 32 samples) and \LB{} minus \IPW{} in Countdown-Code (1000 problems $\times$ 8 samples), in all answers and among parseable answers. Per-problem paired bootstrap 95\% intervals; the last row gives the mean over seeds and the seed-level $p$.}
\label{tab:heldout}
\setlength{\tabcolsep}{6pt}
\begin{tabular}{lcccc}
\toprule
Seed & LeetCode $\Delta H$ & LeetCode $\Delta C$ & Countdown-Code $\Delta C$ & $\Delta C$, parseable \\
\midrule
1 & \pmci{ho3.lc.s1.LBRB.H} & \pmci{ho3.lc.s1.LBRB.C} & \pmci{ho3.cc.s1.LBIPW.C} & \pmci{ho3.cc.LBIPW.s1.Cparse} \\
2 & \pmci{ho3.lc.s2.LBRB.H} & \pmci{ho3.lc.s2.LBRB.C} & \pmci{ho3.cc.s2.LBIPW.C} & \pmci{ho3.cc.LBIPW.s2.Cparse} \\
3 & \pmci{ho3.lc.s3.LBRB.H} & \pmci{ho3.lc.s3.LBRB.C} & \pmci{ho3.cc.s3.LBIPW.C} & \pmci{ho3.cc.LBIPW.s3.Cparse} \\
\midrule
Mean ($p$) & \nm{ho3.lc.LBRB.mean.H} (\nm{ho3.lc.LBRB.p.H}) & \nm{ho3.lc.LBRB.mean.C} (\nm{ho3.lc.LBRB.p.C}) & \nm{ho3.cc.LBIPW.mean.C} (\nm{ho3.cc.LBIPW.p.C}) & \nm{ho3.cc.LBIPW.mean.Cparse} \\
\bottomrule
\end{tabular}
\end{table}

\begin{table}[h]
\centering
\small
\caption{Single-seed arms (seed 1). AUC and rates in \%. Top: LeetCode, fixed count of 8 groups per step or full audit. Bottom: Countdown-Code, $p=0.28$.}
\label{tab:single}
\begin{tabular}{llcccr}
\toprule
& Arm & AUC$(h)$ & $h$, 31--40 & $C$, 31--40 & Labels \\
\midrule
\multirow{6}{*}{LeetCode} & \NONE{} & \nm{lc.NONE.auc} & \nm{lc.NONE.hlate} & \nm{lc.NONE.Clate} & \nm{lc.NONE.labels} \\
& \FULL{}, $c=b$ & \nm{lc.FULL.auc} & \nm{lc.FULL.hlate} & \nm{lc.FULL.Clate} & \nm{lc.FULL.labels} \\
& \textsc{Pos}, $c=b$ & \nm{lc.POS.auc} & \nm{lc.POS.hlate} & \nm{lc.POS.Clate} & \nm{lc.POS.labels} \\
& \textsc{Low}, $c=b$ & \nm{lc.LOW.auc} & \nm{lc.LOW.hlate} & \nm{lc.LOW.Clate} & \nm{lc.LOW.labels} \\
& \textsc{Low}, $c=2b-a$ & \nm{lc.LOWm.auc} & \nm{lc.LOWm.hlate} & \nm{lc.LOWm.Clate} & \nm{lc.LOWm.labels} \\
& \LB{}, $c=2b-a$, $p=0.33$ & \nm{lc.LB.s1.auc} & \nm{lc.LB.s1.hlate} & \nm{lc.LB.s1.Clate} & \nm{lc.LB.s1.labels} \\
\midrule
\multirow{13}{*}{Countdown-Code} & \LB{}, $c=2b-a$ & \nm{cc.LB.s1.auc} & \nm{cc.LB.s1.hlate} & \nm{cc.LB.s1.Clate} & \nm{cc.LB.s1.labels} \\
& \RB{}, $c=2b-a$ & \nm{cc.RB.auc} & \nm{cc.RB.hlate} & \nm{cc.RB.Clate} & \nm{cc.RB.labels} \\
& \LB{}, no std & \nm{cc.LBnostd.auc} & \nm{cc.LBnostd.hlate} & \nm{cc.LBnostd.Clate} & \nm{cc.LBnostd.labels} \\
& \IPW{}, $c=2b-a$ & \nm{cc.IPW.s1.auc} & \nm{cc.IPW.s1.hlate} & \nm{cc.IPW.s1.Clate} & \nm{cc.IPW.s1.labels} \\
& \IPW{}, $c=2b-a$, group std & \nm{cc.IPWstd.auc} & \nm{cc.IPWstd.hlate} & \nm{cc.IPWstd.Clate} & \nm{cc.IPWstd.labels} \\
& \IPW{}, $c=b$ & \nm{cc.IPWtrue.auc} & \nm{cc.IPWtrue.hlate} & \nm{cc.IPWtrue.Clate} & \nm{cc.IPWtrue.labels} \\
& Group advantage correction & \nm{cc.advIPW.auc} & \nm{cc.advIPW.hlate} & \nm{cc.advIPW.Clate} & \nm{cc.advIPW.labels} \\
& \quad allocated by pass count & \nm{cc.optIPW.auc} & \nm{cc.optIPW.hlate} & \nm{cc.optIPW.Clate} & \nm{cc.optIPW.labels} \\
& \FULL{}, $c=2b-a$ & \nm{cc.FULL.s1.auc} & \nm{cc.FULL.s1.hlate} & \nm{cc.FULL.s1.Clate} & \nm{cc.FULL.s1.labels} \\
& \PAC{} & \nm{cc.PAC.auc} & \nm{cc.PAC.hlate} & \nm{cc.PAC.Clate} & \nm{cc.PAC.labels} \\
& \PAC{} without projection & \nm{cc.RawA.auc} & \nm{cc.RawA.hlate} & \nm{cc.RawA.Clate} & \nm{cc.RawA.labels} \\
& \textsc{Cag} (oracle selection) & \nm{cc.CAG.auc} & \nm{cc.CAG.hlate} & \nm{cc.CAG.Clate} & \nm{cc.CAG.labels} \\
& \textsc{Ao} (audited groups only) & \nm{cc.AO.auc} & \nm{cc.AO.hlate} & \nm{cc.AO.Clate} & \nm{cc.AO.labels} \\
\bottomrule
\end{tabular}
\end{table}

\section{Group std normalization and other baselines in Countdown-Code}
\label{app:std}

Adding group std normalization to \IPW{} breaks its unbiasedness and, in our run, strongly attenuated its correction, consistent with \cref{prop:ipwstd}: hacking plateaus at \nm{cc.IPWstd.hlate}\% instead of \nm{cc.IPW.s1.hlate}\%. \LB{} does not depend on std normalization, and in Countdown-Code it trained better without it. Over three seeds, \LB{} without std normalization suppressed hacking as well (late hacking at most \nm{cc.LBnostd.s3.hlate}\%) and ended with higher late correctness than \LB{} with it in every seed, by \nm{cc.LBnostdLB.min.dClate.abs}--\nm{cc.LBnostdLB.max.dClate.abs}\pp{}, with higher pass@8 in every seed. The cost is \nm{cc.LBnostdLB.s3.labelratio.abs}\% to \nm{cc.LBnostdLB.s1.labelratio.abs}\% more labels and \nm{cc.LBnostdLB.s1.tokenratio.abs}\% to \nm{cc.LBnostdLB.s3.tokenratio.abs}\% more generated tokens, because responses grow longer. On held-out problems, the no-std \LB{} checkpoints of seeds 2 and 3 answer \nm{ho3.nostd.s2.LBnostd.C}\% and \nm{ho3.nostd.s3.LBnostd.C}\% correctly (\nm{ho3.nostd.s2.LBnostdLB.C.abs} and \nm{ho3.nostd.s3.LBnostdLB.C.abs}\pp{} above \LB{} with std normalization, and \nm{ho3.nostd.s2.LBnostdIPW.C.abs} and \nm{ho3.nostd.s3.LBnostdIPW.C.abs}\pp{} above \IPW{} at the same seeds), with a truncation tradeoff: in seed 3, \nm{ho3.nostd.s3.LBnostd.M}\% of held-out responses were unparseable, almost all after reaching the 2048-token limit, compared with \nm{ho3.cc.s3.LB.M}\% for std-normalized \LB{}. Removing std normalization changes GRPO's weighting of easy and hard prompts \citep{liu2025drgrpo}; \LB{} composes with it. When \LB{} runs without std normalization, its late correctness exceeds that of \IPW{} at the same seed by \nm{cc.LBnostdIPW.s2.dClate.abs} and \nm{cc.LBnostdIPW.s3.dClate.abs}\pp{} in seeds 2 and 3, so the gap between \LB{} and \IPW{} is not an artifact of the normalization.

\PAC{} \citep{moya2026verifier} keeps the weak reward, so each GRPO step still reinforces hacks, and the projected correction is locally orthogonal to the weak-reward gradient in the chosen metric. With per-item audits at the same rate, hacking rose to \nm{cc.PAC.hlate}\% in the late window, and the unprojected correction reached \nm{cc.RawA.hlate}\% (\cref{tab:single}). The oracle selection rule (\textsc{Cag}) matched \LB{} on hacking and correctness, including held-out accuracy, with \nm{cc.CAG.labels} labels against \nm{cc.LB.s1.labels}, mainly because it skips groups whose corrective gap is zero. Learning only from audited groups (\textsc{Ao}) matched \LB{} on hacking and correctness but drifted toward long responses (\nm{cc.AO.lenlate} against \nm{cc.LB.s1.lenlate} tokens on average in the late window); \nm{ho.cc.AO.M}\% of its held-out responses gave no answer, almost all because they reached the 2048-token limit.

\section{The audit unit at equal marginals}
\label{app:unit}

\begin{table}[h]
\centering
\small
\caption{LeetCode, independent per-item audits minus random whole-group audits at the same audit rate $p=0.4$ (group std normalization). D0: penalty 0, 32 prompts per step, a fixed number of audited groups per step ($12$ plus a Bernoulli$(0.8)$ draw). D4: penalty $-11.5$, 16 prompts per step, each group audited with probability $p$. Paired bootstrap 95\% intervals.}
\label{tab:unit}
\begin{tabular}{llccc}
\toprule
& Seed & $\Delta\mathrm{AUC}(h)$ & $\Delta h$, steps 31--40 & $\Delta C$, steps 31--40 \\
\midrule
\multirow{3}{*}{D0} & 1 & \pmci{lc.D0.s1.dauc} & \pmci{lc.D0.s1.dhlate} & \pmci{lc.D0.s1.dClate} \\
& 2 & \pmci{lc.D0.s2.dauc} & \pmci{lc.D0.s2.dhlate} & \pmci{lc.D0.s2.dClate} \\
& 3 & \pmci{lc.D0.s3.dauc} & \pmci{lc.D0.s3.dhlate} & \pmci{lc.D0.s3.dClate} \\
\midrule
\multirow{2}{*}{D4} & 1 & \pmci{lc.D4.s1.dauc} & \pmci{lc.D4.s1.dhlate} & \pmci{lc.D4.s1.dClate} \\
& 2 & \pmci{lc.D4.s2.dauc} & \pmci{lc.D4.s2.dhlate} & \pmci{lc.D4.s2.dClate} \\
\bottomrule
\end{tabular}
\end{table}

The two experiments also differ in prompts per step and in how the number of audited groups is drawn, so they are not a strict two-factor design. Spreading heavy penalties over many groups saturates the normalized advantages of every group with a caught hack (\cref{prop:ipwstd}), which also shrinks the positive advantage of the uncaught hacks beside it, whereas random whole-group audits leave the unaudited groups untouched. At penalty $-2.5$ and $p=0.33$, per-item random audits also failed to suppress hacking, with an AUC difference against \RB{} of \nm{lc.IBRB.s1.dauc}, \nm{lc.IBRB.s2.dauc} and \nm{lc.IBRB.s3.dauc}\pp{} over the three seeds; we have not tested per-item low-pass targeting.

\paragraph{An exact comparison for one group.} Fix a group with $k$ weak passes, of which $h\ge1$ are hacks and the rest correct, and failures at reward $b$, and let $\Psi(j)$ be its normalized hack mass $\frac1K\sum_{i\in H}A_i$ when $j$ of the $h$ hacks are caught. Auditing the whole group with probability $p$ gives $E_G=(1-p)\Psi(0)+p\Psi(h)$; auditing each weak pass independently with probability $p$ catches $X\sim\mathrm{Binomial}(h,p)$ hacks and gives $E_I=\sum_j\binom hjp^j(1-p)^{h-j}\Psi(j)$, at the same expected cost. Since $\E[X]=hp$,
\[
E_I-E_G=\sum_{j=1}^{h-1}\Pr(X=j)\Big[\Psi(j)-\Big(1-\frac jh\Big)\Psi(0)-\frac jh\Psi(h)\Big],
\]
so per-item audits lower the hack mass more than whole-group audits exactly when $\Psi$ at intermediate depths lies, on average, below the chord between its end points; with centered advantages $\Psi$ is affine in $j$ and the two agree. For $K=16$, $h=8$, $k=10$, $(a,b)=(3.5,0.5)$ and $p=0.4$, penalty 0 gives $E_G=\nm{unit.c0.EG}$ and $E_I=\nm{unit.c0.EI}$ (whole groups better), while penalty $-11.5$ gives $E_G=\nm{unit.c115.EG}$ and $E_I=\nm{unit.c115.EI}$ (single items better).

The D0 comparison was first run with 16 prompts per step and a random number of audited groups; its three seeds gave $-4.18$, $-0.41$ and $+3.02$\pp{}, too noisy to decide, and we preregistered the 32-prompt version above before running it. The seed-level mean of the 32-prompt version is \nm{lc.D0.mean.dauc}\pp{} ($p=\nm{lc.D0.p.dauc}$).

\section{Hack attempts in LeetCode}
\label{app:attempts}

Following the grading harness of \citet{wong2026monitor}, we define the attempt rate as the union of the three evaluator categories in which a response defines or modifies \texttt{run\_tests}: strict hack, correct with an attempt, and failed attempt. \Cref{tab:attempts} reports, over steps 31--40 of each arm, the shares of all responses that are strict hacks, correct, correct with an attempt, correct without one, and attempts. The categories overlap and correctness keeps its usual definition; the attempt label is an observable property of the response, not a measure of intent.

\begin{table}[h]
\centering
\small
\caption{LeetCode, steps 31--40, shares of all responses (\%). ``Attempt'' counts every response that defines its own tests.}
\label{tab:attempts}
\begin{tabular}{lrrrrr}
\toprule
Arm & Strict H & Correct & Correct, attempt & Correct, no attempt & Attempt \\
\midrule
\LB{}, seed 1 & \nm{att.LB.s1.late.H} & \nm{att.LB.s1.late.C} & \nm{att.LB.s1.late.C_att} & \nm{att.LB.s1.late.C_pure} & \nm{att.LB.s1.late.attempt} \\
\LB{}, seed 2 & \nm{att.LB.s2.late.H} & \nm{att.LB.s2.late.C} & \nm{att.LB.s2.late.C_att} & \nm{att.LB.s2.late.C_pure} & \nm{att.LB.s2.late.attempt} \\
\LB{}, seed 3 & \nm{att.LB.s3.late.H} & \nm{att.LB.s3.late.C} & \nm{att.LB.s3.late.C_att} & \nm{att.LB.s3.late.C_pure} & \nm{att.LB.s3.late.attempt} \\
\RB{}, seed 1 & \nm{att.RB.s1.late.H} & \nm{att.RB.s1.late.C} & \nm{att.RB.s1.late.C_att} & \nm{att.RB.s1.late.C_pure} & \nm{att.RB.s1.late.attempt} \\
\RB{}, seed 2 & \nm{att.RB.s2.late.H} & \nm{att.RB.s2.late.C} & \nm{att.RB.s2.late.C_att} & \nm{att.RB.s2.late.C_pure} & \nm{att.RB.s2.late.attempt} \\
\RB{}, seed 3 & \nm{att.RB.s3.late.H} & \nm{att.RB.s3.late.C} & \nm{att.RB.s3.late.C_att} & \nm{att.RB.s3.late.C_pure} & \nm{att.RB.s3.late.attempt} \\
\textsc{Ib}, seed 1 & \nm{att.IB.s1.late.H} & \nm{att.IB.s1.late.C} & \nm{att.IB.s1.late.C_att} & \nm{att.IB.s1.late.C_pure} & \nm{att.IB.s1.late.attempt} \\
\textsc{Ib}, seed 2 & \nm{att.IB.s2.late.H} & \nm{att.IB.s2.late.C} & \nm{att.IB.s2.late.C_att} & \nm{att.IB.s2.late.C_pure} & \nm{att.IB.s2.late.attempt} \\
\textsc{Ib}, seed 3 & \nm{att.IB.s3.late.H} & \nm{att.IB.s3.late.C} & \nm{att.IB.s3.late.C_att} & \nm{att.IB.s3.late.C_pure} & \nm{att.IB.s3.late.attempt} \\
\FULL{} ($c=b$), seed 1 & \nm{att.FULLb.s1.late.H} & \nm{att.FULLb.s1.late.C} & \nm{att.FULLb.s1.late.C_att} & \nm{att.FULLb.s1.late.C_pure} & \nm{att.FULLb.s1.late.attempt} \\
\FULL{} ($c=b$), seed 2 & \nm{att.FULLb.s2.late.H} & \nm{att.FULLb.s2.late.C} & \nm{att.FULLb.s2.late.C_att} & \nm{att.FULLb.s2.late.C_pure} & \nm{att.FULLb.s2.late.attempt} \\
\FULL{} ($c=b$), seed 3 & \nm{att.FULLb.s3.late.H} & \nm{att.FULLb.s3.late.C} & \nm{att.FULLb.s3.late.C_att} & \nm{att.FULLb.s3.late.C_pure} & \nm{att.FULLb.s3.late.attempt} \\
\FULL{} ($c=2b-a$), seed 1 & \nm{att.FULLr.s1.late.H} & \nm{att.FULLr.s1.late.C} & \nm{att.FULLr.s1.late.C_att} & \nm{att.FULLr.s1.late.C_pure} & \nm{att.FULLr.s1.late.attempt} \\
\FULL{} ($c=2b-a$), seed 2 & \nm{att.FULLr.s2.late.H} & \nm{att.FULLr.s2.late.C} & \nm{att.FULLr.s2.late.C_att} & \nm{att.FULLr.s2.late.C_pure} & \nm{att.FULLr.s2.late.attempt} \\
\FULL{} ($c=2b-a$), seed 3 & \nm{att.FULLr.s3.late.H} & \nm{att.FULLr.s3.late.C} & \nm{att.FULLr.s3.late.C_att} & \nm{att.FULLr.s3.late.C_pure} & \nm{att.FULLr.s3.late.attempt} \\
\NONE{} & \nm{att.NONE.late.H} & \nm{att.NONE.late.C} & \nm{att.NONE.late.C_att} & \nm{att.NONE.late.C_pure} & \nm{att.NONE.late.attempt} \\
\FULL{} ($c=b$, pilot) & \nm{att.FULL.late.H} & \nm{att.FULL.late.C} & \nm{att.FULL.late.C_att} & \nm{att.FULL.late.C_pure} & \nm{att.FULL.late.attempt} \\
\textsc{Pos}, $c=b$ & \nm{att.POS.late.H} & \nm{att.POS.late.C} & \nm{att.POS.late.C_att} & \nm{att.POS.late.C_pure} & \nm{att.POS.late.attempt} \\
\textsc{Low}, $c=b$ & \nm{att.LOW.late.H} & \nm{att.LOW.late.C} & \nm{att.LOW.late.C_att} & \nm{att.LOW.late.C_pure} & \nm{att.LOW.late.attempt} \\
\textsc{Low}, $c=2b-a$ & \nm{att.LOWm.late.H} & \nm{att.LOWm.late.C} & \nm{att.LOWm.late.C_att} & \nm{att.LOWm.late.C_pure} & \nm{att.LOWm.late.attempt} \\
\bottomrule
\end{tabular}
\end{table}

Under the reflected penalty, both low-pass designs remove the attempts along with the strict hacks. Auditing every weak pass with the strong-verdict reward $c=b$ leaves strict hacking at \nm{att.FULLb.s2.late.H}--\nm{att.FULLb.s1.late.H}\% and attempts at \nm{att.FULLb.s2.late.attempt}--\nm{att.FULLb.s1.late.attempt}\% across the three seeds (pooled attempt rate \nm{lc.FULLb.pool.attempt}\%, exceeding strict hacking by \nm{lc.FULLb.pool.residue}\pp{}), whereas auditing every weak pass with $2b-a$ eliminates strict hacking in all three seeds and reduces attempts to at most \nm{att.FULLr.s1.late.attempt}\%; the difference in late attempts has the same sign in every seed but varies widely across seeds. That correct responses with attempts reinforce the behavior was observed by \citet{wong2026monitor}.

\paragraph{Held-out problems by difficulty.} We sampled 16 responses on each of 353 held-out problems that were never used for RL training, from the starting checkpoint and from the final \LB{} and \RB{} adapters of all three seeds, with the training sampling settings. Problems were split by the dataset's difficulty label and into thirds of the starting checkpoint's per-problem correctness (\cref{tab:difficulty}; bootstrap intervals resample problems conditional on the trained adapters and fixed strata). \RB{}'s strict hacking rises from the easiest to the hardest third by \nm{diff.RBgap.emp.H}\pp{} \ci{diff.RBgap.emp.H}, while its attempt rate stays between \nm{diff.RB.label_hard.attempt}\% and \nm{diff.RB.emp_easiest.attempt}\% in every stratum; strict hacking thus tracks how often an attempt coincides with a failed solution. \LB{} and \FULL{} ($c=2b-a$) show no resurgence on hard problems. Under \FULL{} ($c=b$), strict hacking is \nm{diff.FULLb.emp_easiest.H}\% on the easiest third with \nm{diff.FULLb.emp_easiest.attempt}\% attempts, and rises to \nm{diff.FULLb.emp_hardest.H}\% on the hardest third (\nm{diff.FULLbgap.emp.H}\pp{} \ci{diff.FULLbgap.emp.H}).

\begin{table}[h]
\centering
\small
\caption{Held-out LeetCode problems by difficulty (\%; \LB{} and \RB{} pooled over three seeds). Thirds are defined by the starting checkpoint's correctness on each problem.}
\label{tab:difficulty}
\begin{tabular}{llrrr}
\toprule
Model & Stratum (problems) & Strict H & Attempt & Correct \\
\midrule
Start & hardest third (\nm{diff.size.emp_hardest}) & \nm{diff.start.emp_hardest.H} & \nm{diff.start.emp_hardest.attempt} & \nm{diff.start.emp_hardest.C} \\
Start & easiest third (\nm{diff.size.emp_easiest}) & \nm{diff.start.emp_easiest.H} & \nm{diff.start.emp_easiest.attempt} & \nm{diff.start.emp_easiest.C} \\
\LB{} & hardest third & \nm{diff.LB.emp_hardest.H} & \nm{diff.LB.emp_hardest.attempt} & \nm{diff.LB.emp_hardest.C} \\
\LB{} & easiest third & \nm{diff.LB.emp_easiest.H} & \nm{diff.LB.emp_easiest.attempt} & \nm{diff.LB.emp_easiest.C} \\
\LB{} & labelled hard (\nm{diff.size.label_hard}) & \nm{diff.LB.label_hard.H} & \nm{diff.LB.label_hard.attempt} & \nm{diff.LB.label_hard.C} \\
\RB{} & hardest third & \nm{diff.RB.emp_hardest.H} & \nm{diff.RB.emp_hardest.attempt} & \nm{diff.RB.emp_hardest.C} \\
\RB{} & easiest third & \nm{diff.RB.emp_easiest.H} & \nm{diff.RB.emp_easiest.attempt} & \nm{diff.RB.emp_easiest.C} \\
\RB{} & labelled hard & \nm{diff.RB.label_hard.H} & \nm{diff.RB.label_hard.attempt} & \nm{diff.RB.label_hard.C} \\
\FULL{} ($c=b$) & hardest third & \nm{diff.FULLb.emp_hardest.H} & \nm{diff.FULLb.emp_hardest.attempt} & \nm{diff.FULLb.emp_hardest.C} \\
\FULL{} ($c=b$) & easiest third & \nm{diff.FULLb.emp_easiest.H} & \nm{diff.FULLb.emp_easiest.attempt} & \nm{diff.FULLb.emp_easiest.C} \\
\FULL{} ($c=2b-a$) & hardest third & \nm{diff.FULLr.emp_hardest.H} & \nm{diff.FULLr.emp_hardest.attempt} & \nm{diff.FULLr.emp_hardest.C} \\
\bottomrule
\end{tabular}
\end{table}

\section{The online rule and a one-step diagnosis at step 120}
\label{app:diag}

The online rule estimates, from the audits' own labels and a smoothed estimate of composition by pass count, how many labels to spend in each group, and may audit only part of a group. It improved $J$ over \LB{} in simulation on frozen panels. From the main starting point it fell behind \LB{} early (area under the hacking curve \nm{lc.ON80_LB.dauc.abs}\pp{} higher) before both reached near zero. The two arms of a pair did not see identical first-step samples, so the pilot is not a controlled counterfactual for the cause.

At the starting checkpoint of the step-120 pilot, we drew \nm{diag.batches} common batches of 16 prompts with 16 responses each from the training distribution, assigned audits to the same responses under \LB{} and under the frozen online rule (penalty $-2.5$, $p=0.2$), and computed the exact projections of each response's score gradient on estimated directions $d_H=\nabla P(H)$ and $d_C=\nabla P(C)$. The directions were estimated by score-function averages with leave-one-out baselines on 864 other prompts with four samples each at temperature 1; their cosine is \nm{diag.cos}. KL, clipping, gradient clipping and weight decay were omitted.

Online minus \LB{}, the hack mass $S_H$ was \nm{diag.dSH} $\pm$ \nm{diag.dSH.se} (standard error over batches), and the projection of the first-order SGD direction on $d_H$ was \nm{diag.dsgdH} $\pm$ \nm{diag.dsgdH.se}. These batches were sampled as in training, with top-$p$ 0.95, but their scores were evaluated under the full softmax, so the score of a sample does not have mean zero: every response class had a large positive mean projection on $d_H$ (\nm{diag2.proj.min} to \nm{diag2.proj.max}). Splits of the difference by response class, or into the change in the std denominator and the direct audit correction, are dominated by this shared component. In these batches the hack responses contributed \nm{diag.dfromH} $\pm$ \nm{diag.dfromH.se} and the other responses \nm{diag.dfromRest} $\pm$ \nm{diag.dfromRest.se}, as if the other responses cancelled a stronger correction through the hacks, and the denominator and direct parts were \nm{diag2.reweight} $\pm$ \nm{diag2.reweight.se} and \nm{diag2.direct} $\pm$ \nm{diag2.direct.se}. We repeated the comparison on \nm{diag.batches} new batches sampled with top-$p$ 1, where the class means ranged from \nm{diagp1.proj.min} to \nm{diagp1.proj.max}. The hack mass again differed by \nm{diagp1.dSH} $\pm$ \nm{diagp1.dSH.se} and the projection by \nm{diagp1.dsgdH} $\pm$ \nm{diagp1.dsgdH.se}, but the hack responses contributed \nm{diagp1.dfromH} $\pm$ \nm{diagp1.dfromH.se} and the other responses \nm{diagp1.dfromRest} $\pm$ \nm{diagp1.dfromRest.se}, and the denominator and direct parts were \nm{diagp1.reweight} $\pm$ \nm{diagp1.reweight.se} and \nm{diagp1.direct} $\pm$ \nm{diagp1.direct.se}. Under both samplers the lower hack mass of the online rule did not produce a detectable difference in the projection, and the attribution to response classes or to the two parts changed with the sampler, so we do not attribute the outcome to any of them. The errors include only the variation over batches, not the error of the estimated directions or of the frozen online state.

Randomly permuting the advantages among the hacks of each group leaves every class coefficient unchanged. Under the online rule the mean change in the $d_H$ projection was \nm{diag.perm.mean}, as it must be by symmetry, but its root mean square over \nm{diag.perm.pairs} assignments was \nm{diag.perm.rms}, \nm{diag.perm.ratio} times the mean net projection (\nm{diag.online.net}); in the batches sampled with top-$p$ 1 it was \nm{diagp1.perm.rms}, \nm{diagp1.perm.ratio} times the mean net projection (\nm{diagp1.online.net}). Under \LB{}, which audits whole groups, it was exactly zero. Which hack receives the penalty is therefore a large source of noise in a single update, although its role in the long-run outcome is not identified. We also computed the first step of a fresh Adam state; its projections were too noisy to compare, and since Adam is nonlinear and the omitted terms do not cancel after it, this diagnosis does not identify the optimizer's contribution.

\section{Preregistration and deviations}
\label{app:prereg}

Each comparison was entered in a dated log, with its arms, primary metric and decision rule, and committed to version control before its runs started. Predictions and their outcomes:
\begin{itemize}
\item The per-item random audit comparison in LeetCode was preregistered without a predicted direction and is reported with two-sided intervals and a seed-level $t$-test (\cref{app:stats}).
\item The online pilot was preregistered as a test of predictions from frozen-state simulation; the predictions failed.
\item For the Countdown-Code runs of \RB{} and of the group-level advantage correction, we predicted from the sign of the frozen-panel hack mass that \RB{} would not suppress hacking (it did not) and that the group-level correction would (its late hacking of \nm{cc.advIPW.hlate}\% fell just short of our 2\% threshold; both later seeds suppressed hacking).
\item For seeds 2 and 3 of the LeetCode group-level advantage correction we predicted suppression from the frozen-panel sign; seed 3 suppressed hacking, and seed 2 ended with \nm{lc.advIPW.s2.hlate}\% late hacking (partial).
\item LeetCode low budgets: from the sign of the frozen-panel hack mass we predicted that \LB{} would not suppress hacking at 0.1 and would at 0.2 (both held in two seeds), and that \LB{} with corrections scaled by 3.1 and the group-level advantage correction would suppress it at 0.1 (both held). For \LB{} with penalty $-11.5$ at 0.1 we predicted no suppression; it suppressed hacking for about 15 steps and then rose again (partial). A continuation from its step 40 rose to \nm{lc.LBp10hc40.hlate}\% in steps 31--40, against the prediction of below 10\% from the frozen-panel sign at that state.
\item In total, before 16 training runs we predicted from the sign of the frozen-panel hack mass whether hacking would be suppressed (late hacking below 2\%) or not (10\% or more); 13 predictions held. Besides the runs above, the 16 include seed 1 of the LeetCode group-level advantage correction, the Countdown-Code correction allocated by pass count, and a single-seed LeetCode run of \LB{} with penalty $-0.25$ instead of $-2.5$; all three were predicted to suppress hacking and did. The three misses (seed 1 of the Countdown-Code group-level correction, seed 2 of the LeetCode one, and \LB{} with penalty $-11.5$ at 0.1) all ended between the two thresholds and nearer the predicted one: the first two, predicted to suppress hacking, ended at \nm{cc.advIPW.hlate}\% and \nm{lc.advIPW.s2.hlate}\%, and the third, predicted not to, ended at \nm{lc.LBp10h.s1.hlate}\% after suppressing hacking for about 15 steps. The continuation of the last was preregistered with a pair of thresholds outside this count, and its frozen-panel prediction failed (above). Applied to the starting panels of 23 further runs, the same rule matched 20 (\cref{sec:boundaries}).
\end{itemize}

\section{Statistical details of the main-text comparisons}
\label{app:stats}

Training intervals are paired bootstrap 95\% intervals over (step, prompt) units, held-out intervals over problems; seed-level tests are one-sample $t$-tests on the three per-seed differences. Differences are in percentage points.
\begin{itemize}
\item \LB{} minus \RB{}, LeetCode (\cref{tab:lbrb}): area under the hacking curve \nm{lc.LBRB.mean.d_auc} on average (seed-level $t(2)=\nm{lc.LBRB.t.d_auc}$, $p=\nm{lc.LBRB.p.d_auc}$); late correctness \nm{lc.LBRB.s1.dClate}, \nm{lc.LBRB.s2.dClate} and \nm{lc.LBRB.s3.dClate} (mean \nm{lc.LBRB.mean.d_c}, $p=\nm{lc.LBRB.p.d_c}$); labels \nm{lc.LBRB.mean.labelratio}\% of \RB{}'s ($p=\nm{lc.LBRB.p.labelratio}$).
\item Random items minus \LB{}, LeetCode: area under the hacking curve \pmci{lc.IBLB.s1.dauc}, \pmci{lc.IBLB.s2.dauc} and \pmci{lc.IBLB.s3.dauc} (mean \nm{lc.IBLB.mean.d_auc}, $t(2)=\nm{lc.IBLB.t.d_auc}$, $p\nm{lc.IBLB.p.d_auc}$); late correctness \pmci{lc.IBLB.s1.dClate}, \pmci{lc.IBLB.s2.dClate} and \pmci{lc.IBLB.s3.dClate} (mean \nm{lc.IBLB.mean.d_c}, $p=\nm{lc.IBLB.p.d_c}$). Random items minus \RB{}: \pmci{lc.IBRB.s1.dauc}, \pmci{lc.IBRB.s2.dauc} and \pmci{lc.IBRB.s3.dauc}.
\item Single-seed LeetCode arms (\cref{tab:single}): area under the hacking curve, \textsc{Low} minus \textsc{Pos} \pmci{lc.LOW_POS.dauc}, \textsc{Low}$_{2b-a}$ minus \textsc{Low} \pmci{lc.LOWm_LOW.dauc}; late correctness, \textsc{Low}$_{2b-a}$ minus \FULL{} \pmci{lc.LOWm_FULL.dClate}.
\item Full audits, $c=b$ minus $c=2b-a$, LeetCode: late attempts \pmci{lc.FULLbFULLr.s1.datt}, \pmci{lc.FULLbFULLr.s2.datt} and \pmci{lc.FULLbFULLr.s3.datt} (mean \nm{lc.FULLbFULLr.mean.d_att}, seed-level $p=\nm{lc.FULLbFULLr.p.d_att}$).
\item \LB{} minus \IPW{}, Countdown-Code (\cref{tab:lbipw,tab:heldout}): area under the hacking curve mean \nm{cc.LBIPW.mean.dauc} ($p=\nm{cc.LBIPW.p.dauc}$); late correctness $t(2)=\nm{cc.LBIPW.t.dClate}$, $p=\nm{cc.LBIPW.p.dClate}$; held-out accuracy $t(2)=\nm{ho3.cc.LBIPW.t.C}$, $p=\nm{ho3.cc.LBIPW.p.C}$. No-std \LB{} minus \IPW{}, late correctness: \pmci{cc.LBnostdIPW.s2.dClate} and \pmci{cc.LBnostdIPW.s3.dClate} (seeds 2 and 3).
\item Countdown-Code single-seed comparisons with \LB{} (seed 1): \RB{} minus \LB{}, area under the hacking curve \pmci{cc.RBLB.dauc}, late correctness \pmci{cc.RBLB.dClate}; held-out, \LB{} minus \RB{}, accuracy \pmci{ho.cc.LBRB.dC} and among parseable answers \pmci{ho.cc.LBRB.dC_parseable}.
\item Countdown-Code group-level advantage correction minus \LB{} of the same seed (seeds 1, 2, 3): area under the hacking curve \pmci{cc.advIPWLB.s1.dauc}, \pmci{cc.advIPWLB.s2.dauc}, \pmci{cc.advIPWLB.s3.dauc}; late correctness \pmci{cc.advIPWLB.s1.dClate}, \pmci{cc.advIPWLB.s2.dClate}, \pmci{cc.advIPWLB.s3.dClate}. Held-out, \LB{} minus the correction: accuracy \pmci{ho.cc.LBadvIPW.s1.dC}, \pmci{ho.cc.LBadvIPW.s2.dC}, \pmci{ho.cc.LBadvIPW.s3.dC}; unparseable answers \pmci{ho.cc.LBadvIPW.s1.dM}, \pmci{ho.cc.LBadvIPW.s2.dM}, \pmci{ho.cc.LBadvIPW.s3.dM}; accuracy among parseable answers \pmci{ho.cc.LBadvIPW.s1.dC_parseable}, \pmci{ho.cc.LBadvIPW.s2.dC_parseable}, \pmci{ho.cc.LBadvIPW.s3.dC_parseable}.
\item Countdown-Code held-out accuracy with cut answers continued to 4096 tokens: \LB{} minus the group-level advantage correction \pmci{ho4k.cc.LBadvIPW.s1.dC}, \pmci{ho4k.cc.LBadvIPW.s2.dC}, \pmci{ho4k.cc.LBadvIPW.s3.dC}; \LB{} minus per-item \IPW{} \pmci{ho4k.cc.LBIPW.s1.dC}, \pmci{ho4k.cc.LBIPW.s2.dC}, \pmci{ho4k.cc.LBIPW.s3.dC} (seeds 1, 2, 3).
\item Countdown-Code full audit with the reflected penalty minus \LB{} (seeds 1, 2, 3): area under the hacking curve \pmci{cc.FULLLB.s1.dauc}, \pmci{cc.FULLLB.s2.dauc}, \pmci{cc.FULLLB.s3.dauc}; late correctness \pmci{cc.FULLLB.s1.dClate}, \pmci{cc.FULLLB.s2.dClate}, \pmci{cc.FULLLB.s3.dClate}; held-out accuracy, \LB{} minus the full audit, \pmci{ho.cc.LBFULL.s1.dC}, \pmci{ho.cc.LBFULL.s2.dC}, \pmci{ho.cc.LBFULL.s3.dC}; \nm{cc.FULL.s1.labels}, \nm{cc.FULL.s2.labels} and \nm{cc.FULL.s3.labels} labels. Correction allocated by pass count (seed 1): minus \LB{}, area under the hacking curve \pmci{cc.optIPWLB.dauc} and late correctness \pmci{cc.optIPWLB.dClate}; minus the uniform correction, late correctness \pmci{cc.optIPWadvIPW.dClate}; held-out, \LB{} minus it \pmci{ho.cc.LBoptIPW.dC}, it minus the uniform correction \pmci{ho.cc.optIPWadvIPW.dC}, unparseable answers \pmci{ho.cc.optIPWadvIPW.dM}.
\item LeetCode low budgets (seed 1): \LB{} with corrections scaled by 3.1 at 0.1 minus \LB{} at 0.33, area under the hacking curve \pmci{lc.LBscale10LBm.dauc}, late correctness \pmci{lc.LBscale10LBm.dClate}; minus the group-level advantage correction at 0.1, area under the hacking curve \pmci{lc.LBscale10advIPWp10.dauc}, late correctness \pmci{lc.LBscale10advIPWp10.dClate}; \nm{lc.LBscale10.labels}, \nm{lc.advIPWp10.labels} and \nm{lc.LBp10.s1.labels} labels for the scaled \LB{}, the advantage correction and \LB{} at 0.1.
\item LeetCode group-level advantage correction minus \LB{} of the same seed (seeds 1, 2, 3): area under the hacking curve \pmci{lc.advIPW.s1.dauc}, \pmci{lc.advIPW.s2.dauc}, \pmci{lc.advIPW.s3.dauc} (the three \LB{} seeds range from \nm{lc.LB.s2.auc}\% to \nm{lc.LB.s3.auc}\%); late correctness \pmci{lc.advIPW.s1.dClate}, \pmci{lc.advIPW.s2.dClate}, \pmci{lc.advIPW.s3.dClate}; held-out accuracy on 353 training-split problems \pmci{lc.advIPW.s1.hoC}, \pmci{lc.advIPW.s2.hoC}, \pmci{lc.advIPW.s3.hoC}; late hacking \nm{lc.advIPW.s1.hlate}\%, \nm{lc.advIPW.s2.hlate}\% and \nm{lc.advIPW.s3.hlate}\%; \nm{lc.advIPW.s1.labels}, \nm{lc.advIPW.s2.labels} and \nm{lc.advIPW.s3.labels} labels against \nm{lc.advIPW.s1.labels.LB}, \nm{lc.advIPW.s2.labels.LB} and \nm{lc.advIPW.s3.labels.LB}.
\item Group std normalization, Countdown-Code: \IPW{} with minus without it, area under the hacking curve \pmci{cc.IPWstd_IPW.dauc}. No-std \LB{} minus \LB{}: late correctness \nm{cc.LBnostdLB.s1.dClate}, \nm{cc.LBnostdLB.s2.dClate} and \nm{cc.LBnostdLB.s3.dClate}, each interval above zero (mean \nm{cc.LBnostdLB.mean.dClate}, $p=\nm{cc.LBnostdLB.p.dClate}$); held-out accuracy in seeds 2 and 3, \pmci{ho3.nostd.s2.LBnostdLB.C} and \pmci{ho3.nostd.s3.LBnostdLB.C}, and against \IPW{} \pmci{ho3.nostd.s2.LBnostdIPW.C} and \pmci{ho3.nostd.s3.LBnostdIPW.C}.
\item Oracle selection (\textsc{Cag}) minus \LB{}, held-out accuracy: \pmci{ho.cc.CAG-LB.C}.
\item Online rule minus \LB{}, LeetCode, single seed: area under the hacking curve \pmci{lc.ON120_LB120.dauc} from step 120 and \pmci{lc.ON80_LB.dauc} from step 80.
\end{itemize}

\end{document}